\documentclass[11pt]{article}

\usepackage{amsmath,amsthm,amssymb,amsfonts,bbm,bm}
\usepackage{cite,epsfig,epsf,url,enumitem}
\usepackage{etoolbox}
\usepackage{booktabs}
\usepackage{graphicx}
\usepackage{hyperref}
\usepackage{fullpage}
\usepackage{float}

\usepackage[small,bf]{caption}
\newtheorem{theorem}{Theorem}[section]

\newtheorem{lemma}[theorem]{Lemma}
\newtheorem{corollary}[theorem]{Corollary}
\theoremstyle{definition}
\newtheorem{definition}[theorem]{Definition}
\newtheorem{assumption}[theorem]{Assumption}
\theoremstyle{remark}
\newtheorem{remark}[theorem]{Remark}

\title{Stability and Generalization of Quantum Neural Networks}

\author{%
    Jiaqi Yang, Wei Xie\thanks{Corresponding author.}, Xiaohua Xu\footnotemark[1] \\
    University of Science and Technology of China, Hefei 230027, China \\
    yangjiaqi@mail.ustc.edu.cn, xxieww@ustc.edu.cn, xiaohuaxu@ustc.edu.cn
}

\date{}

\begin{document}

\maketitle

\begin{abstract}
Quantum neural networks (QNNs) play an important role as an emerging technology in the rapidly growing field of quantum machine learning. While their empirical success is evident, the theoretical explorations of QNNs, particularly their generalization properties, are less developed and primarily focus on the uniform convergence approach. In this paper, we exploit an advanced tool in classical learning theory, i.e., algorithmic stability, to study the generalization of QNNs. We first establish high-probability generalization bounds for QNNs via uniform stability. Our bounds shed light on the key factors influencing the generalization performance of QNNs and provide practical insights into both the design and training processes. We next explore the generalization of QNNs on near-term noisy intermediate-scale quantum (NISQ) devices, highlighting the potential benefits of quantum noise. Moreover, we argue that our previous analysis characterizes worst-case generalization guarantees, and we establish a refined optimization-dependent generalization bound for QNNs via on-average stability. Numerical experiments on various real-world datasets support our theoretical findings.
\end{abstract}

\section{Introduction}
Quantum machine learning (QML) is a rapidly growing field that has generated great excitement \cite{biamonte2017quantum}. With the aim of solving complex problems beyond the reach of classical computers, firm and steady progress has been achieved over the past decade \cite{harrow2017quantum,dunjko2018machine,ciliberto2018quantum}. Quantum neural network (QNN), or equivalently, the parameterized quantum circuit (PQC) with a classical optimizer, has received great attention thanks to its potential to achieve quantum advantages on near-term noisy intermediate scale quantum (NISQ) devices \cite{farhi2018classification,havlivcek2019supervised,beer2020training}. Typically, stochastic gradient descent (SGD) is employed as the classical optimizer in QNNs due to its simplicity and efficiency. Driven by the significance of understanding the power of QNNs, a growing body of literature has been conducted to investigate their expressivity \cite{sim2019expressibility,bravo2020scaling,wu2021expressivity,herman2023expressivity}, trainability \cite{mcclean2018barren,cerezo2021cost,sharma2022trainability,larocca2023theory}, and generalization \cite{caro2020pseudo,banchi2021generalization,bu2022statistical,du2022efficient,caro2022generalization,gyurik2023structural}. Investigating the generalization of QNNs is crucial for understanding their underlying working principles and capabilities from a theoretical perspective. Nevertheless, to date, the theoretical establishment in this area is still in its infancy. 

In classical learning theory, a variety of techniques for generalization analysis are known. One of the most popular approach is via the uniform convergence analysis, which studies the uniform generalization gaps in a hypothesis space. This approach typically uses complexity measures such as VC-dimension \cite{vapnik1998statistical}, covering numbers \cite{cucker2007learning}, Rademacher complexity \cite{bartlett2002rademacher} to develop capacity-dependent bounds. To the best of our knowledge, essentially all generalization bounds derived for QNNs so far are of the uniform kind, which, however, are not sufficient to explain the generalization behavior of current-scale QNNs \cite{gil2024understanding}. Impressive alternatives have been proposed, including sample compression \cite{littlestone1986relating}, PAC-Bayes \cite{mcallester1998some}, and algorithmic stability \cite{bousquet2002stability}. In particular, the algorithmic stability tools have advantages in some aspects, such as  dimension-independent and adaptability for broad learning paradigms. Therefore, it is natural to explore the generalization of QNNs through the lens of algorithmic stability.

In this paper, we study the generalization of QNNs trained using the SGD algorithm based on algorithmic stability. Our contributions are summarized as follows.
\begin{itemize}
    \item We analyze the uniform stability of QNNs trained using the SGD algorithm. Our results reveal that the negative effects arising from complex QNN models on stability can be mitigated by setting appropriate step sizes. We investigate the simplified stability results for two commonly used step sizes. 
    \item We establish high-probability generalization bounds for QNNs via uniform stability. Our bounds shed light on the key factors influencing the generalization performance of QNNs. Furthermore, we provide practical insights into both the design and training processes for developing powerful QNNs. Notably, our bounds help explain why over-parameterized QNNs trained by SGD exhibit excellent generalization, which was not explained by existing bounds for QNNs with the same setting \cite{du2022efficient,caro2022generalization}.
    \item We further extend our analysis to the noise scenario. Considering the standard depolarizing noise model (easily extended to other noise models), we similarly establish high-probability generalization bounds for QNNs. Our bounds highlight the potential benefits of quantum noise. Specifically, we provide new insight that the noise naturally occurring in quantum devices can be effectively tuned as a form of regularization. Moreover, our results theoretically explain the effectiveness of the method proposed in \cite{somogyi2024method}, which enhances QNN performance by controlling the noise level in quantum hardware as hyperparameters during the training process.
    \item We argue that previous analysis characterizes worst-case generalization guarantees. As a complement to our previous results, we establish a refined optimization-dependent generalization bound via on-average stability. This new bound reveals that the worst-case bounds can be improved in certain regimes. In addition, we corroborate the connection of our bound to the generalization performance of the recent experiments in \cite{gil2024understanding}. Our bound captures the effect of randomized labels on generalization in terms of the on-average variance of SGD. As a corollary, our results validates the intuition that if we are good at the initialization point, the QNN model is more stable and thus generalizes better.
    \item We conduct numerical experiments on real-world datasets, and the empirical results indeed support our theoretical findings.
\end{itemize}

\section{Related Work} 
\textbf{Generalization analysis of QNNs.} The theory of generalization for QNNs is less developed and primarily focuses on the uniform convergence approach. Recent studies on the generalization of QNNs typically employ complexity measures such as pseudo-dimension \cite{caro2020pseudo}, effective dimension \cite{abbas2021power}, VC-dimension \cite{gyurik2023structural}, Rademacher complexity \cite{bu2021rademacher,bu2022statistical,bu2023effects}, covering number \cite{du2022efficient,caro2022generalization}, and all their uniform relatives. Of the most interest to ours is \cite{du2022efficient}, where the authors provide generalization bounds for QNNs based on covering number, all derived with the same setting as ours. Later, \cite{caro2022generalization} uses quantum channels to derive more general results. Their generalization bounds exhibit a sublinear dependence on the number of trainable quantum gates. Unfortunately, this limits their results in providing meaningful guarantees for over-parameterized QNNs. More recently, it has been argued in \cite{gil2024understanding} that the traditional measures of model complexity are not sufficient to explain the generalization behavior of current-scale QNNs. Their empirical findings highlight the need to shift perspective toward non-uniform generalization measures in QML. Therefore, it is natural to leverage the concept of algorithmic stability  \cite{bousquet2002stability}, which additionally enjoys desirable properties on flexibility (dimension-independence) and adaptivity (suiting for diverse learning scenarios). While completing our work, we became aware of an \href{https://openreview.net/pdf?id=lirR6Wfkd6}{anonymous work} on OpenReview that derives a generalization bound for data re-uploading QNNs via uniform stability. We develop our approach independently of them, which differs in multiple key aspects and thus obtains more general results. Unlike their focus on constant step sizes, we also analyze the \emph{practically relevant} scenarios of decaying step sizes. We further explore the generalization of QNNs in the noise scenario, highlighting the potential benefits of quantum noise. Moreover, we introduce another concept of on-average stability to provide the refined optimization-generalization bound.


\textbf{Stability and Generalization.}
The classical framework of quantifying generalization via stability was established in an influential paper \cite{bousquet2002stability}, where the celebrated concept of uniform stability was introduced. Subsequently, the uniform stability measure was extended to study stochastic algorithms \cite{elisseeff2005stability}. The seminal work \cite{hardt2016train} pioneered the generalization analysis of SGD via uniform stability, which inspired several follow-up studies to understand stochastic optimization algorithms based on different algorithmic stability measures, e.g., on-average stability \cite{kuzborskij2018data,lei2020fine}, locally elastic stability \cite{deng2021toward}, and argument stability \cite{lei2020fine,bassily2020stability}. Stability-based generalization analysis was also developed for transfer learning \cite{kuzborskij2018data}, pairwise learning \cite{lei2021generalization,yang2021simple}, and minimax problems \cite{lei2021stability,xing2021algorithmic}. The power of stability analysis is especially reflected by its ability to derive optimal generalization bounds in expectation \cite{shalev2010learnability}. Recent studies show that uniform stability can yield almost optimal high-probability bounds \cite{feldman2018generalization,bousquet2020sharper}. While significant progress has been achieved, there is a lack of analysis on generalization from the perspective of stability in the context of QNNs.

\section{Problem Setup}
\subsection{Quantum Computation Basics and Notations} 
We briefly introduce some basic concepts of quantum computation that are necessary for this work. Interested readers are recommended to the celebrated textbook by Nielsen and Chuang \cite{nielsen2001quantum}. Qubit is the fundamental unit of quantum computation and quantum information. An \( N \)-qubit quantum state is represented by the density matrix, which is a  Hermitian, positive semi-definite matrix \( \rho\in\mathbb{C}^{2^N\times2^N}\) with \( \mathrm{Tr}(\rho)=1 \). Quantum gates are unitary matrices used to transform quantum states. Common single-qubit gates include Pauli rotations \( \left\{ R_P(\theta) = \mathrm{e}^{-\mathrm{i} \frac{\theta}{2} P} \mid P \in \left\{ X, Y, Z \right\} \right\} \), which are in the exponential form of Pauli matrices,
\begin{equation}
X = \begin{pmatrix}
0 & 1 \\
1 & 0
\end{pmatrix}, \quad
Y = \begin{pmatrix}
0 & -\textnormal{i} \\
\textnormal{i} & 0
\end{pmatrix}, \quad
Z = \begin{pmatrix}
1 & 0 \\
0 & -1
\end{pmatrix}.
\notag
\end{equation}
Common two-qubit gates include controlled-X gate, \( \textnormal{CNOT} = |0\rangle\langle0|\otimes I+ |1\rangle\langle1| \otimes X \), is used to generate quantum entanglement among qubits. The evolution of a quantum state \( \rho \) can be mathematically described by employing a quantum circuit: \( \rho^\prime=U(\boldsymbol{\theta})\rho U^\dagger(\boldsymbol{\theta})\), where the unitary \( U(\boldsymbol{\theta})\) is usually parameterized by a series of single-qubit rotation gate angles \( \boldsymbol{\theta}\) and basic two-qubit gates, and \( U^\dagger \) denotes the conjugate transpose of \( U \). Quantum measurement is a means to extract classical (observable) information from the quantum states. An observable is represented by a Hermitian matrix \( O\in\mathbb{C}^{2^N\times 2^N} \). Since measurement is an irreversible process, it is typically introduced at the end of the quantum circuit. The expectation of the output is calculated as \( \mathrm{Tr}(O\rho^\prime) \).

\textbf{Quantum Neural Network.} The quantum neural network typically contains three parts, i.e., an \( N \)-qubit quantum circuit \( U(\boldsymbol{\theta}) \), an observable \( O\in\mathbb{C}^{2^N\times 2^N}\), and a classical optimizer that update trainable parameters \( \boldsymbol{\theta} \) to minimize the predefined objective function. For classical data \( \boldsymbol{x} \), the input is first encoded into a quantum state via a map \( \boldsymbol{x}\mapsto\rho(\boldsymbol{x})\in\mathbb{C}^{2^N\times 2^N} \). Define \( U(\boldsymbol{\theta}) = \prod_{k=1}^{K_g} U_k(\boldsymbol{\theta}) \), where \( U_k \) refers to the \( k \)-th quantum gate. In general, \( U(\boldsymbol{\theta}) \) is formed by \( K \) trainable gates, \( K_g - K \) fixed gates, and \( \boldsymbol{\theta} \in \mathbb{R}^K \). Under the above definitions, the output function of the QNN can be written as follows
\begin{equation*}
    \begin{aligned}
         f_{\boldsymbol{\theta}}(\boldsymbol{x}) = \mathrm{Tr}\left( O U(\boldsymbol{\theta}) \rho(\boldsymbol{x}) U^\dagger(\boldsymbol{\theta})\right).
    \end{aligned}
\end{equation*}
In addition, the quantum gates in NISQ chips are prone to errors \cite{preskill2018quantum}. The noise can be simulated by applying certain quantum channels to each quantum gate, specifically the depolarization channel. Given a quantum state \( \rho\in\mathbb{C}^{2^N\times2^N} \), the depolarization channel \(\mathcal{N}_p\) acts on a \(2^N\)-dimensional Hilbert space follows \( \mathcal{N}_p(\rho)=(1-p)\rho+p\mathbb{I}/2^N\), where \( p\) is the noise level and \( \mathbb{I}/2^N\) is the maximally mixed state. 

\subsection{Stability and Generalization Analysis of Randomized Algorithm}
Let \( \mathcal{D} \) be a probability distribution defined on a sample space \( \mathcal{Z}=\mathcal{X}\times\mathcal{Y} \), where \( \mathcal{X}\subseteq \mathbb{R}^d \) is an input space and \( \mathcal{Y}\subseteq\mathbb{R} \) is an output space. We quantify the loss of \( \boldsymbol{\theta} \) on a single example \( \boldsymbol{z}=(\boldsymbol{x},y )\) by \( \ell\left(\boldsymbol{\theta};\boldsymbol{z}\right) \). The objective is to learn a model \( \boldsymbol{\theta}\in\mathbb{R}^K \) minimizing the \emph{population risk} defined by
\begin{equation*}
    R_{\mathcal{D}}(\boldsymbol{\theta})=\mathbb{E}_{\boldsymbol{z}\sim\mathcal{D}}\left[\ell\left(\boldsymbol{\theta};\boldsymbol{z}\right)\right].
\end{equation*}
In practice, we do not know the distribution \( \mathcal{D} \) but instead have access to a dataset \( S=\{\boldsymbol{z}_i=(\boldsymbol{x}_i, y_i):i=1,\ldots,m\} \) independently drawn from \( \mathcal{D} \). Then, we approximate \( R_{\mathcal{D}} \) by \emph{empirical risk}
\begin{equation*}
    R_S(\boldsymbol{\theta})=\frac{1}{m}\sum_{i=1}^m\ell\left(\boldsymbol{\theta};\boldsymbol{z}_i\right).
\end{equation*}
For a randomized algorithm \( A \), denote by \( A(S) \) its output model based on the training dataset \( S \). Let \( R(\boldsymbol{\theta}_*)=\inf_{\boldsymbol{\theta}} R_{\mathcal{D}}(\boldsymbol{\theta}) \) and \( R(\boldsymbol{\theta}_*^S)=\inf_{\boldsymbol{\theta}} R_{S}(\boldsymbol{\theta}) \), then the \emph{excess risk} is \( \mathbb{E}_A\left[R_{\mathcal{D}}\left(A\left(S\right)\right)-R(\boldsymbol{\theta}_*)\right] \). Here, the expectation \( \mathbb{E}_A[\cdot] \) is taken only over the internal randomness of A. It can be decomposed as
\begin{equation*}
    \begin{aligned}
       \mathbb{E}_A\left[R_{\mathcal{D}}\left(A\left(S\right)\right)-R(\boldsymbol{\theta}_*)\right] 
       &\leq \mathbb{E}_A\left[R_{\mathcal{D}}\left(A\left(S\right)\right)-R_{S}\left(A\left(S\right)\right)\right] + \mathbb{E}_A\left[R_{S}\left(A\left(S\right)\right)-R_S(\boldsymbol{\theta}_*^S)\right],
    \end{aligned}
\end{equation*}
where we have used the fact that  \( R_S(\boldsymbol{\theta}_*^S)\leq R_S(\boldsymbol{\theta}_*) \)  by the definition of \( \boldsymbol{\theta}_*^S \). The first term is called the \emph{generalization (error) gap}, as it quantifies the generalization shift from training to testing behavior. The second term is called the \emph{optimization error}, as it measures how effectively the algorithm minimizes empirical risk. This paper focuses on bounding the generalization gap, for which a popular approach is based on the stability analysis of the algorithm. 

\textbf{Algorithmic Stability.} Algorithmic stability plays an important role in classical learning theory, which measures the sensitivity of an algorithm to the perturbation of training sets. Our analysis relies on two widely used stability measures, namely \emph{uniform stability} \cite{bousquet2002stability,hardt2016train} and \emph{on-average stability} \cite{lei2020fine}. Below, we recall these notions.

\begin{definition}[Uniform Stability] \label{def:Uniform Stability}
    We say a randomized algorithm \( A \) is \( \epsilon \)-uniformly stable if for any datasets \( S,S^\prime\in\mathcal{Z}^m \) that differ by at most one example, we have
    \begin{equation*}
        \sup_{\boldsymbol{z}} \left| \mathbb{E}_A \left[ \ell\left(A(S);\boldsymbol{z}\right) - \ell\left(A({S^\prime});\boldsymbol{z}\right) \right] \right| \leq \epsilon.
    \end{equation*}
\end{definition}
\begin{definition}[On-Average Stability]\label{def:On-Average Stability}
    Let \( S=\{\boldsymbol{z}_1,\ldots,\boldsymbol{z}_m\} \) and \( S^\prime=\{\boldsymbol{z}_1^\prime,\ldots,\boldsymbol{z}_m^\prime\} \) be drawn independently from \( \mathcal{D} \). For any \(i\in[m]\), define \(S^{(i)}=\{\boldsymbol{z}_1,\ldots,\boldsymbol{z}_{i-1},\boldsymbol{z}_i^\prime,\boldsymbol{z}_{i+1},\ldots,\boldsymbol{z}_m\}\) as the set formed from \( S \) by replacing the \( i \)-th element with \( \boldsymbol{z}_i^\prime\). We say a randomized algorithm \( A \) is on-average \( \epsilon \)-stable if
    \begin{equation*}
        \mathbb{E}_{S,S^\prime,A}\left[\frac{1}{m}\sum_{i=1}^m\|A(S)-A(S^{(i)})\|_2\right]\leq\epsilon.
    \end{equation*}
\end{definition}
The celebrated relationship between algorithmic stability and generalization was established in the following lemma.
\begin{lemma}[Stability and Generalization]\label{lem:Stability and Generalization}
    Let \( A \) be a randomized algorithm, \( \epsilon>0 \) and \( \delta\in(0,1) \). 
    \begin{enumerate}
        \item[(a)] If \( A \) is \( \epsilon \)-uniformly stable, then the expected generalization gap satisfies
        \begin{equation*}
            \left|\mathbb{E}_{S,A} \left[ R_{\mathcal{D}}\left(A(S)\right) - R_S\left(A(S)\right) \right]\right|\leq\epsilon.
        \end{equation*}
        \item[(b)] Assume \( \ell(\boldsymbol{\theta};\boldsymbol{z})\in[0,M] \) for all \( \boldsymbol{\theta}\in\mathbb{R}^K \) and \( \boldsymbol{z}\in\mathcal{Z} \). If \( A \) is \( \epsilon \)-uniformly stable, then with probability at least \( 1-\delta \), the generalization gap satisfies
        \begin{equation*}
            \begin{aligned}
                &\mathbb{E}_A \left[ R_{\mathcal{D}}\left(A(S)\right) - R_S\left(A(S)\right) \right] =\mathcal{O}\left(\epsilon\log m \log(1/\delta) + Mm^{-\frac{1}{2}}\sqrt{ \log(1/\delta)}\right).
            \end{aligned}
        \end{equation*}
        \item[(c)] Assume \( \left|\ell(\boldsymbol{\theta}_1;\boldsymbol{z})-\ell(\boldsymbol{\theta}_2;\boldsymbol{z})\right|\leq L\|\boldsymbol{\theta}_1-\boldsymbol{\theta}_2\|_2\) for all \( \boldsymbol{\theta}_1,\boldsymbol{\theta}_2\in\mathbb{R}^K \) and \( \boldsymbol{z}\in\mathcal{Z} \). If \( A \) is on-average \( \epsilon \)-stable, then the expected generalization gap satisfies 
        \begin{equation*}
            \begin{aligned}
                 &\left|\mathbb{E}_{S,A} \left[ R_{\mathcal{D}}\left(A(S)\right) - R_S\left(A(S)\right) \right]\right| \leq L\epsilon.
            \end{aligned}
        \end{equation*}
    \end{enumerate}
\end{lemma}
\begin{remark}    
    Part (a) and Part (b) establish the connection between uniform stability and generalization \cite{bousquet2002stability,bousquet2020sharper}. Part (c) establishes the connection between on-average stability and generalization \cite{lei2020fine}. Both Part (a) and Part (c) provide generalization bounds in expectation. Under the assumption \( \left|\ell(\boldsymbol{\theta}_1;\boldsymbol{z})-\ell(\boldsymbol{\theta}_2;\boldsymbol{z})\right|\leq L\|\boldsymbol{\theta}_1-\boldsymbol{\theta}_2\|_2\) for all \( \boldsymbol{\theta}_1,\boldsymbol{\theta}_2\in\mathbb{R}^K \) and \( \boldsymbol{z}\in\mathcal{Z} \), we notice that \( \epsilon\)-uniformly stable implies at least \( (\epsilon/L) \)-on-average stable. Thus, we can recover the worst-case generalization bound as in Part (a). Part (b) provides an almost optimal high-probability generalization bound.
\end{remark}
Typically,  SGD is employed as the classical optimizer in QNNs. Given a training dataset \( S\in\mathcal{Z}^m \), the QNN minimizes the following objective function,  
\begin{equation*}
    \min_{\boldsymbol{\theta}}\frac{1}{m}\sum_{i=1}^m\ell\left(f_{\boldsymbol{\theta}}(\boldsymbol{x_i});y_i\right).
\end{equation*}
\begin{definition}[Stochastic Gradient Descent]
    Let \( \boldsymbol{\theta}_0\in\mathbb{R}^K \) be an initial point. SGD updates \( \{\boldsymbol{\theta}_t\} \) as follows
    \begin{equation*}
        \boldsymbol{\theta}_{t+1}=\boldsymbol{\theta}_t-\eta_t\nabla_{\boldsymbol{\theta}}\ell\left(f_{\boldsymbol{\theta}_t}(\boldsymbol{x}_{i_t});y_{i_t}\right),
    \end{equation*}
\end{definition}
where \( \eta_t \) is the step size, and \( \boldsymbol{z}_{i_t }=(\boldsymbol{x}_{i_t},y_{i_t}) \) is the sample chosen in iteration \( t \). There are two popular schemes for choosing the example indices \( i_t \). One is to choose uniformly from \( \{1,\ldots,m\} \) at each step. The other is to choose a random permutation of \( \{1,\ldots,m\} \) and then process the examples in order. Our results hold for both variants.
\subsection{Main Assumptions} 
We aim to analyze the stability and generalization of QNNs trained using the SGD algorithm. To achieve this, we first introduce the necessary assumptions. 
\begin{assumption}[\( \alpha_\ell\)-Lipschitzness]\label{ass:Lipschitzness}
    We assume that the loss function \( \ell \) is \( \alpha_\ell \)-Lipschitz, if for all \({\boldsymbol{\theta}_1},\boldsymbol{\theta}_2\in\mathbb{R}^K\), and \( \boldsymbol{z}\in\mathcal{Z} \), we have
    \begin{equation*}
        \left| \ell\left(f_{\boldsymbol{\theta}_1}(\boldsymbol{x});y\right) - \ell\left(f_{\boldsymbol{\theta}_2}(\boldsymbol{x});y\right) \right|
        \leq \alpha_\ell \left| f_{\boldsymbol{\theta}_1}(\boldsymbol{x}) - f_{\boldsymbol{\theta}_2}(\boldsymbol{x}) \right|.
    \end{equation*}
\end{assumption}
\begin{assumption}[\( \nu_\ell\)-smoothness]\label{ass:smoothness}
    We assume that the loss function \( \ell \) is \( \nu_\ell \)-smooth, if for all \({\boldsymbol{\theta}_1},\boldsymbol{\theta}_2\in\mathbb{R}^K\), and \( \boldsymbol{z}\in\mathcal{Z} \), we have
    \begin{equation*}
        \left| 
        \frac{\partial}{\partial f} \ell\left(f_{\boldsymbol{\theta}_1}(\boldsymbol{x});y\right) 
        - 
        \frac{\partial}{\partial f} \ell\left(f_{\boldsymbol{\theta}_2}(\boldsymbol{x});y\right) 
        \right| 
        \leq \nu_\ell \left| f_{\boldsymbol{\theta}_1}(\boldsymbol{x}) - f_{\boldsymbol{\theta}_2}(\boldsymbol{x}) \right|.
    \end{equation*}
\end{assumption}
\begin{remark}
    Unlike the classic setup of \cite{hardt2016train}, we define Lipschitzness and smoothness with respect to the output function \( f_{\boldsymbol{\theta}}(\cdot) \) rather than the parameters \( \boldsymbol{\theta} \). By relaxing their assumptions, we provide a fine-grained analysis of stability and generalization in the context of QNNs.
\end{remark}

\section{Main Results}
In this section, we present our main results on the stability and generalization bounds for QNNs trained using SGD algorithm. We note that our results are highly general and cover diverse ansatz of QNNs, as long as it is composed of the parameterized single-qubit gates and two-qubit gates. Due to space limitations, please refer to the Appendix \ref{ap: Appendix C} for detailed theoretical proofs.

\subsection{Uniform Stability-Based Generalization Bounds}
In this subsection, we analyze the uniform stability of QNNs and subsequently derive the high-probability generalization bounds.
\begin{theorem}[Uniform Stability Bound]\label{the:1}
    Suppose that Assumption \ref{ass:Lipschitzness} and \ref{ass:smoothness} hold. Let \( A(S) \) be the QNN model trained on the dataset \( S\in\mathcal{Z}^m \) using the SGD algorithm with step sizes \( \eta_t\) for \( T \) iterations, then \( A(S) \) is \( \epsilon \)-uniformly stable with 
    \begin{equation*} 
        \epsilon\leq\sum_{t=0}^{T-1}\left[\prod_{j=t+1}^{T-1}(1+\eta_j\kappa)\right]\frac{2\sqrt{2}\eta_t\alpha_\ell^2K\|O\|^2}{m},
    \end{equation*}
    where
    \begin{equation}\label{eq:k}
         \kappa:=\alpha_\ell K \|O\| + \sqrt{2} \nu_\ell K \|O\|^2.
    \end{equation}
\end{theorem}
\begin{remark}
    The stability bound expands with coefficient \( 1+\eta_t\kappa \), where \( \kappa \) depends on the spectral norm of the observable \( \|O\| \) and the number of trainable quantum gates \( K \). In practice, \( \|O\| \) is typically bounded such that \( 0\leq \|O\|\leq 1\). Moreover, the negative effects of complex QNN models with large \( K \) on stability can be mitigated by setting appropriate step sizes. 
\end{remark}
Notice that the step size is a vital parameter in Theorem \ref{the:1}. We further investigate the simplified stability results for two commonly used step sizes in the following corollary.
\begin{corollary}\label{cor:1}
    Suppose that Assumption \ref{ass:Lipschitzness} and \ref{ass:smoothness} hold, and \( \ell(\cdot,\cdot)\in[0,M] \). Let \( A(S) \) be the QNN model trained on the dataset \( S\in\mathcal{Z}^m \) using the SGD algorithm with step sizes \( \eta_t\) for \( T \) iterations. 
    \begin{enumerate}
        \item [(a)] If we choose the constant step sizes \( \eta_t=\eta \), then \( A(S) \) is \( \epsilon \)-uniformly stable with
        \begin{equation*}
            \epsilon\leq\frac{2\sqrt{2}\alpha_\ell^2K\|O\|^2}{\kappa m}(1+\eta\kappa)^T.
        \end{equation*}
        \item [(b)] If we choose the monotonically non-increasing step sizes \( \eta_t\leq c/(t+1) \), \( c>0 \), then \( A(S) \) is \( \epsilon \)-uniformly stable with
        \begin{equation*}
            \epsilon \leq \frac{1 + 1/c\kappa}{m}M^{\frac{c\kappa}{c\kappa + 1}} 
            \left( 2\sqrt{2}c\alpha_\ell^2 K \|O\|^2 \right)^{\frac{1}{c\kappa + 1}}
            T^{\frac{c\kappa}{c\kappa + 1}},
        \end{equation*}       
    \end{enumerate}
    where \( \kappa \) is defined by equation (\ref{eq:k}).
\end{corollary}
\begin{remark}
    Part (a) and Part (b) consider constant and decaying step sizes, respectively. The constant step size setting is widely used in prior works on the stability and generalization analysis of classical graph convolutional networks (GCNs) \cite{verma2019stability, ng2024stability, yang2024deeper}. However, the resulting stability bound exhibits an exponential dependence on \( T \). The exponential dependence can be eliminated by using decaying step sizes, with two main approaches for setting the decay. One approach is the classic analysis from \cite{hardt2016train}, which considers the timing of encountering different examples in datasets. It employs a polynomially decaying step sizes \( \eta_t=\mathcal{O}(t^{-1})\). The other approach controls the smallest eigenvalues of the loss’s Hessian matrix, which allows larger step sizes \( \eta_t=\mathcal{O}(t^{-\beta}),\beta\in(0,1) \) \cite{richards2021learning,richards2021stability,lei2022stability}. However, their analysis is limited to shallow neural networks and requires an additional restrictive condition that the network must be sufficiently wide. To develop a highly general stability bound for QNNs, we adopt the same strategy as \cite{hardt2016train}, setting the step sizes \( \eta_t\leq c/(t+1),c>0\) in Part (b).
\end{remark}
By combining Lemma \ref{lem:Stability and Generalization}, Part (b), we can now easily derive the generalization bound. The high-probability generalization bounds for QNNs are subsequently established in the following theorem.
\begin{theorem}[Generalization Bound]\label{the:2}
    Suppose that Assumption \ref{ass:Lipschitzness} and \ref{ass:smoothness} hold, and \( \ell(\cdot,\cdot)\in[0,M] \). Let \( A(S) \) be the QNN model trained on the dataset \( S\in\mathcal{Z}^m \) using the SGD algorithm with step sizes \( \eta_t\) for \( T \) iterations. 
    \begin{enumerate}
        \item [(a)] if we choose the constant step sizes \( \eta_t=\eta \), then the following generalization bound of \(
        A(S) \) holds with probability at least \( 1-\delta \) for \( \delta\in(0,1) \),
        \begin{equation*}
            \begin{aligned}
                &\mathbb{E}_A\left[R_{\mathcal{D}}\left(A(S)\right)-R_S\left(A(S)\right)\right] \leq \mathcal{O}\left(\frac{(1+\eta \kappa)^T}{m}\log m \log(\frac{1}{\delta})+M\sqrt{\frac{\log(\frac{1}{\delta})}{m}}\right),
            \end{aligned}
        \end{equation*}
        \item [(b)] if we choose the monotonically non-increasing step sizes \( \eta_t\leq c/(t+1) \), \( c>0 \), then the following generalization bound of \(
        A(S) \) holds with probability at least \( 1-\delta \) for \( \delta\in(0,1) \),
        \begin{equation*}
            \begin{aligned}
                &\mathbb{E}_A\left[R_{\mathcal{D}}\left(A(S)\right)-R_S\left(A(S)\right)\right] \leq \mathcal{O}\left(\frac{T^{\frac{c\kappa}{c\kappa+1}}}{m}\log m \log(\frac{1}{\delta})+M\sqrt{\frac{\log(\frac{1}{\delta})}{m}}\right),
            \end{aligned}
        \end{equation*} 
    \end{enumerate}
    where \( \kappa \) is defined by equation (\ref{eq:k}).
\end{theorem}
\begin{remark}
    Theorem \ref{the:2} provides practical insights into both the design and training processes for developing powerful QNNs. Specifically, our bounds shed light on the key factors influencing the generalization performance of QNNs, which can be divided into two categories. The first is associated with the design of QNNs, including the number of trainable quantum gates \( K \), and the spectral norm of the observable \( \|O\| \). The second is associated with the training of QNNs, including the step sizes \( \eta_t \), the number of iterations \( T \), and the number of training examples \( m \).
    \begin{itemize}
        \item  \textit{Design.} In practice, \( \|O\| \) is bounded such that \( 0\leq \|O\|\leq 1\), and it is normally chosen as Pauli Strings. The dependence on \( K \) reflects an Occam’s razor principle in the quantum version \cite{blumer1987occam}. It is evident that, a larger \( K \) leads to an increase in the upper bound of the generalization gap. This provides guidance for designing well-performing QNNs with a proper number of trainable quantum gates. Moreover, increasing the number of trainable quantum gates \( K \) directly enhances the expressivity of QNNs \cite{du2021learnability,du2022efficient}. This leads to a trade-off between expressivity and generalization in designing powerful QNNs, analogous to the bias-variance trade-off in classical machine learning. 
        \item \textit{Training.} First, increasing the number of training examples \( m \) directly enhances the generalization performance of QNNs. Second, it is crucial to carefully choose the step sizes. Clearly, the decaying step size setting is a better choice than the constant step size setting. In the constant step size setting, our bound is primarily controlled by the term \( \mathcal{O}\left((1+\eta\kappa)^T/m\right) \). This suggests setting the step size as \( \eta=\mathcal{O}(1/K) \) to ensure \( \eta\kappa=\mathcal{O}(1)\), thereby preventing the generalization bound from becoming vacuous due to the inherent complexity of QNNs. Lastly, we underscore the importance of reducing training time. Decreasing the number of iterations \( T \) is an effective way to reduce the generalization gap. Therefore, in practice, \emph{early stopping} is an important training technique for QNNs, where stop training early after reach a low training error.
    \end{itemize}
\end{remark}
\begin{remark}
    We compare Theorem \ref{the:2} with related works \cite{du2022efficient,caro2022generalization}. The state-of-the-art generalization bound for QNNs in the same setting is of the order \( \mathcal{O}\left(\sqrt{K/m}\right) \). However, the \emph{sublinear} dependence on \( K \) makes their results trivially loose for the \emph{over-parameterized} QNNs, where \( K\gg m \). In contrast, our bounds help explain why over-parameterized QNNs exhibit excellent generalization. Specifically, our bounds highlight that the negative effects of large \( K \) on generalization can be mitigated by setting appropriate step sizes, thereby providing meaningful generalization guarantees as well. As mentioned earlier, in the constant step size setting, our bound suggest setting the step sizes as \( \eta=\mathcal{O}(1/K) \) to balance the negative effect of large \( K \). In the decaying step size setting, our bound is primarily controlled by the term \( \mathcal{O}\left(T^{\frac{c\kappa}{c\kappa+1}}/m\right) \). It is worth noting that the negative effect of \( K \) is significantly reduced. No matter how large \( K \) becomes, the generalization gap converges to zero as \( m\rightarrow{\infty} \), provided that \( T \) is not too large. We show \( T \) can grow as \( m^a\) for a small \( a>1 \). 
\end{remark}
Our previous analysis can be easily extended to the noise scenario. Considering the standard depolarizing noise model, we similarly establish the high-probability generalization bounds for QNNs in the following corollary. We remark that while our results are presented assuming the depolarization noise, they can be easily extended to other noise models.
\begin{corollary}[Generalization Bound Under Depolarizing Noise]\label{cor:2}
    Suppose that Assumption \ref{ass:Lipschitzness} and \ref{ass:smoothness} hold, and \( \ell(\cdot,\cdot)\in[0,M] \). Let \( A(S) \) be the QNN model trained on the dataset \( S\in\mathcal{Z}^m \) using the SGD algorithm with step sizes \( \eta_t\) under depolarizing noise level \( p\in[0,1] \) for \( T \) iterations. 
    \begin{enumerate}
        \item [(a)] if we choose the constant step sizes \( \eta_t=\eta \), then the following generalization bound of \(
        A(S) \) holds with probability at least \( 1-\delta \) for \( \delta\in(0,1) \),
        \begin{equation*}
            \begin{aligned}
                &\mathbb{E}_A\left[R_{\mathcal{D}}\left(A(S)\right)-R_S\left(A(S)\right)\right] \leq \mathcal{O}\left(\frac{\left(1+\left(1-p\right)^{K_g}\eta\kappa\right)^T}{m}\log m \log(\frac{1}{\delta})+M\sqrt{\frac{\log(\frac{1}{\delta})}{m}}\right),
            \end{aligned}
        \end{equation*}
        \item [(b)] if we choose the monotonically non-increasing step sizes \( \eta_t\leq c/(t+1) \), \( c>0 \), then the following generalization bound of \(
        A(S) \) holds with probability at least \( 1-\delta \) for \( \delta\in(0,1) \),
        \begin{equation*}
            \begin{aligned}
                &\mathbb{E}_A\left[R_{\mathcal{D}}\left(A(S)\right)-R_S\left(A(S)\right)\right] \leq \mathcal{O}\left(\frac{T^{\frac{c\kappa}{c\kappa+1/\left(1-p\right)^{K_g}}}}{m}\log m \log(\frac{1}{\delta})+M\sqrt{\frac{\log(\frac{1}{\delta})}{m}}\right),
            \end{aligned}
        \end{equation*} 
    \end{enumerate}
    where \( \kappa \) is defined by equation (\ref{eq:k}).
\end{corollary}
\begin{remark}
    Corollary \ref{cor:2} reveals the potential benefits of the quantum noise. Our bounds shows that an increase in the noise level \( p \) enhances the generalization performance of QNNs. Moreover, previous work has indicated that larger noise leads to poorer expressivity of QNNs \cite{du2021learnability,du2022efficient}. Therefore, our bounds provide new insight that the noise naturally occurring in quantum devices can be effectively tuned as a form of '\emph{quantum regularization}', with the ability balancing expressivity and generalization of QNNs. It has been demonstrated that adding noise to the initial data or the weights of QNNs can have an effect analogous to the technique of \emph{regularization} employed in classical machine learning \cite{nguyen2020quantum,heyraud2022noisy}. Unlike these existing techniques, \cite{somogyi2024method} proposed an approach to control the noise level in quantum hardware as hyperparameters during the training process. They numerically investigate this method for a regression task by using several modeled noise channels and demonstrate an improvement in model performance. Our results theoretically explain the potential reasons for its effectiveness, acting akin to regularization in classical neural networks.
\end{remark}
\subsection{Towards Optimization-Dependent Generalization Bounds}
In the previous subsection, we focused on uniform stability and derived the key results of this paper. However, uniform stability does not depend on the data, but captures only intrinsic characteristics of the learning algorithm and global properties of the objective function. Consequently, previous analysis characterizes worst-case generalization guarantees. As a complement to our previous results, we further investigate a refined optimization-dependent generalization bound, leveraging the less restrictive on-average stability \cite{kuzborskij2018data,lei2020fine}. 

To capture the impact of the variance of the stochastic gradients, we adopt the following standard assumption from stochastic optimization theory \cite{nemirovski2009robust,bottou2010large}. 
\begin{assumption}[Bounded Empirical Variance]\label{ass:Variance}
    For any dataset \( S\in\mathcal{Z}^m \) and \( \boldsymbol{\theta}\in\mathbb{R}^K \), their exist \( \sigma^2>0 \) such that 
    \begin{equation}\label{eq:variance}
        \frac{1}{m}\sum_{i=1}^m\|\nabla_{\boldsymbol{\theta}}\ell\left(f_{\boldsymbol{\theta}}(\boldsymbol{x}_i);\boldsymbol{y}_i\right)-\nabla_{\boldsymbol{\theta}}R_S(\boldsymbol{\theta})\|_2^2\leq\sigma^2.
    \end{equation}
\end{assumption}
\begin{remark}
    Assumption \ref{ass:Variance} essentially bounds the variance of the stochastic gradients for the particular dataset \( S \). Notably, it is always satisfied if \( \|\nabla_{\boldsymbol{\theta}}\ell(f_{\boldsymbol{\theta}}(\boldsymbol{x});\boldsymbol{y})\|_2\leq G\) for any \( \boldsymbol{\theta}\in\mathbb{R}^K\) and \( z\in\mathcal{Z} \), with \( \sigma^2=G^2 \). 
\end{remark}
The following theorem establishes the generalization bound in expectation for QNNs and first links generalization gap to optimization error.
\begin{theorem}[Optimization-Dependent Generalization Bound]\label{the:3}
    Suppose that Assumption \ref{ass:Lipschitzness}, \ref{ass:smoothness}, and \ref{ass:Variance} hold. Let \( A(S) \) be the QNN model trained on the dataset \( S\in\mathcal{Z}^m \) using the SGD algorithm with step sizes \( \eta_t\) for \( T \) iterations, then the expected generalization gap satisfies
    \begin{equation*}
        \begin{aligned}
            & \left|\mathbb{E}_{S,A} \left[ R_{\mathcal{D}}\left(A(S)\right) - R_S\left(A(S)\right) \right]\right| \leq \sum_{t=0}^{T-1}\left[\prod_{j=t+1}^{T-1}(1+\eta_t\kappa)\right]\frac{2\eta_t\alpha_\ell\sqrt{K}\|O\|\left(\mathbb{E}_S[\|\nabla_{\boldsymbol{\theta}} R_S(\boldsymbol{\theta}_t)\|_2]+\sigma\right)}{m},
        \end{aligned}
    \end{equation*}
    where \( \kappa \) is defined by equation (\ref{eq:k}).
\end{theorem}
\begin{remark}
     From Theorem \ref{the:3}, we notice that we can bound \( \sigma \) and the gradient norms \( \|\nabla_{\boldsymbol{\theta}} R_S(\boldsymbol{\theta}_t)\|_2\) by \( G \), where \( G=\sqrt{2}\alpha_\ell\sqrt{K}\|O\|\). Recall from Lemma \ref{lem:Stability and Generalization}, Part (a), uniform stability directly implies generalization in expectation. Thus, we can recover (up to a factor 2), the worst-case generalization bound from Theorem \ref{the:1}. This illustrates well the worst-case notion mentioned earlier, but also the fact that the generalization bound in expectation of Theorem \ref{the:3} can be better than the one of Theorem \ref{the:1}. This will notably be the case in “low noise” regimes, when \( \sigma\ll G\) and the expected gradient norm \( \mathbb{E}_S\|\nabla_{\boldsymbol{\theta}} R_S(\boldsymbol{\theta}_t)\|_2\) reach small values. In particular, the expected gradient norm \( \mathbb{E}_S\|\nabla_{\boldsymbol{\theta}} R_S(\boldsymbol{\theta}_t)\|_2\) is related to the optimization error, which decreases as the parameter \( \boldsymbol{\theta}_t \) minimizes the empirical risk \( R_S(\boldsymbol{\theta}) \).
\end{remark}
\begin{remark}
     Theorem \ref{the:3} helps explain the observations in classification experiments with randomized labels \cite{gil2024understanding}. They conduct randomization experiments by training QNNs on a set of quantum states with varying levels of label corruption. Without changing the QNN ansatz, the number of training examples, or the optimization algorithm, they observe a steady increase in the generalization gap as the random label probability increases. Our bound properly captures how the generalization gap changes with the fraction of random labels via the on-average variance of SGD. Specifically, as the random label probability increases, the on-average variance \( \sigma \) keeps increasing and the generalization gap also increases.
\end{remark}
Note that Theorem \ref{the:3} can be similarly extended to various step size settings and noise scenario. Additionally, we show that the expected gradient norm \( \mathbb{E}_S\left[\|\nabla_{\boldsymbol{\theta}} R_S(\boldsymbol{\theta}_t)\|_2\right] \) are influenced by the choice of the initialization point. 
\begin{lemma}[Link with Initialization Point]\label{lem:Link with Initialization Point}
     Suppose that Assumption \ref{ass:Lipschitzness}, \ref{ass:smoothness}, and \ref{ass:Variance} hold. Let \( A(S) \) be the QNN model trained on the dataset \( S\in\mathcal{Z}^m \) using the SGD algorithm with step sizes \( \eta_t\leq 1/\kappa
     \) for \( T \) iterations, then the following bound holds
     \begin{equation*}
        \begin{aligned}
            &\sum_{t=0}^{T-1}\eta_t\mathbb{E}_S\left[\|\nabla_{\boldsymbol{\theta}} R_S(\boldsymbol{\theta}_t)\|_2\right] \leq 2\sqrt{\left(\sum_{t=0}^{T-1}\eta_t\right)\left(R_S(\boldsymbol{\theta}_0)-R_S(\boldsymbol{\theta}^*)+\frac{\kappa\sigma^2}{2}\sum_{t=0}^{T-1}\eta_t^2\right)},
        \end{aligned}
    \end{equation*}
    where \( \boldsymbol{\theta}^* \) is the empirical risk minimizer of \( R_S(\boldsymbol{\theta})\), \( \kappa \) is defined by equation (\ref{eq:k}).
\end{lemma}
\begin{remark}
    Lemma \ref{lem:Link with Initialization Point} validates the intuition that if we are good at the initialization point \( \boldsymbol{\theta}_0\), the QNN model is more stable and thus generalizes better. It suggests choosing a initialization point with low empirical risk in practice.
\end{remark}

\section{Experimental Evaluation}\label{sec:5}
In this section, we conduct numerical simulations to validate our theoretical results. Note that we do not aim to optimize the test accuracy, but rather a simple, interpretable experimental setup. 

\textbf{Experiments Setup.} We consider two real-world datasets: MNIST \cite{lecun2010mnist} and Fashion MNIST \cite{xiao2017fashion},  which have also been explored in the \href{https://openreview.net/pdf?id=lirR6Wfkd6}{anonymous work} as well as in other QML studies \cite{huang2021power,chen2021end,hur2022quantum}. For these datasets, we conduct binary classification tasks on the digit 0/1 and the T-shirt/Trouser. The implementation of QNN is as follows. The classical information \( \boldsymbol{x} \) is encoded into a quantum state \( \rho(\boldsymbol{x}) \) via angle encoding. We adopt the most widely used hardware-efficient ansatz such that the construction of \( U(\boldsymbol{\theta}) \) follows a layerwise structure using single-qubit Pauli rotation gates and two-qubit CNOT gates. We empirically estimate the generalization gap by calculating the absolute difference between the training and test errors. Each experiment setting is repeated 50 times to obtain statistical results.

\textbf{Number of Trainable Quantum Gates.} We investigate the impact of the number of trainable quantum gates \( K \) on generalization gap by varying the QNN layer numbers \( L=\{4,6,8,10,12\}\). The step size is set to \( \eta=0.01 \). In Figure \ref{fig:1}, we observe that as \( L \) increases, i.e., \( K \) increases, the generalization gap also increases. This observation is consistent with our generalization bounds in Theorem \ref{the:2} regarding the impact of \( K \) on the generalization gap.

\textbf{Step Size.} To avoid introducing additional hyperparameters, we mainly focus on constant step sizes to investigate the impact of the step size on generalization gap. We try various step sizes \( \eta=\{0.001,0.005,0.01,0.05,0.1\} \). The QNN layer number is set to \( L=8 \). It is clear from Figure \ref{fig:2} that the generalization gap increases with larger step sizes. This observation aligns well with our generalization bounds in Theorem \ref{the:2}. Additionally, since the number of trainable quantum gates \( K \) is fixed throughout the experiment, this further validating our theoretical explanation presented after Theorem \ref{the:2}, that is, the negative effects arising from the inherent complexity of the QNN models on generalization can be mitigated by setting appropriate step sizes.

\textbf{Quantum Noise.} We investigate the impact of the quantum noise on generalization gap by varying the noise level \( p=\{0.001,0.005,0.01,0.05,0.1\}\). The QNN layer number is set to \( L=4 \) and the step size is set to \( \eta=0.01 \). It is depicted in Figure \ref{fig:3} that as the noise level \( p \) increases, the generalization gap decreases. This result echoes with Corollary \ref{cor:2} and reinforces the insights that quantum noise can be effectively tuned as a form of regularization.
\begin{figure}[H] 
    \centering
    \centerline{\includegraphics[width=7.1in]{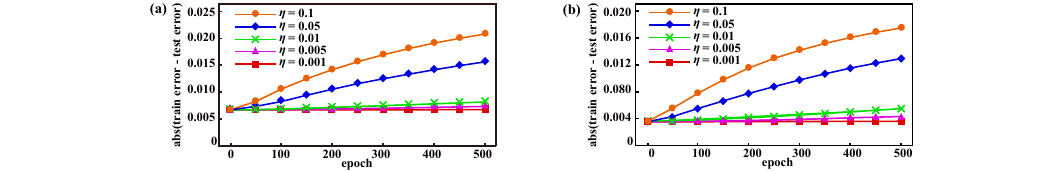}}
    \caption{Generalization gap for varying numbers of trainable quantum gates: (a) MNIST, (b) Fashion MNIST.}
    \label{fig:1}
\end{figure}

\begin{figure}[H]
    \centering
    \centerline{\includegraphics[width=7.1in]{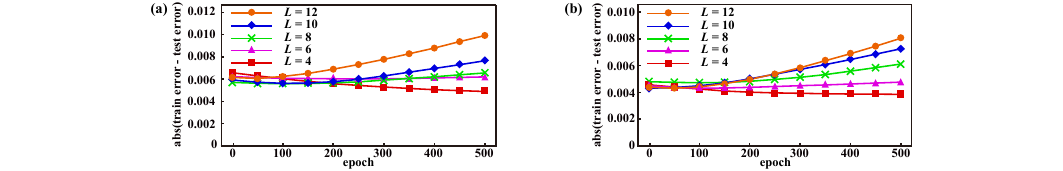}}
    \caption{Generalization gap for varying step sizes: (a) MNIST, (b) Fashion MNIST.}
    \label{fig:2}
\end{figure}

\begin{figure}[H] 
\begin{center}
\centerline{\includegraphics[width=7.1in]{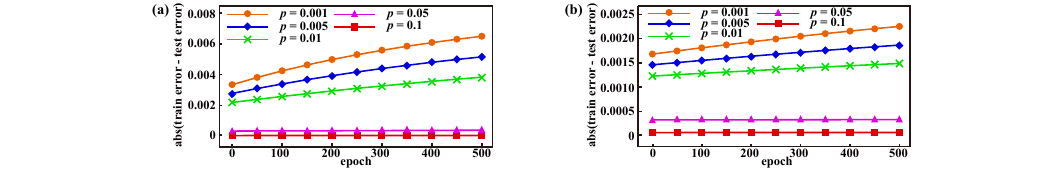}}
\caption{Generalization gap with varying noise levels: (a) MNIST, (b) Fashion MNIST.}
\label{fig:3}
\end{center}
\end{figure}

\textbf{Random Label.} We conduct experiments to validate our explanation for the observations in \cite{gil2024understanding} that a classification dataset with randomized labels can substantially degrade the generalization performance of QNNs. For all the data labels in each dataset, we replace their underlying true labels with random labels with probability \( r \). The QNN layer number is set to \( L=8 \), the step size is set to \( \eta=0.01 \), and the number of iterations is set to \( T = 500 \). In Figure \ref{fig:4}, we present the results under the random label probability \( r=\{0.1,0.2,0.3,0.4,0.5\} \). It can be seen from these results that the on-average variance \( \sigma \) consistently increases as the random label probability \( r \) increases. At the same time, the generalization gap also increases. This validate our theoretical explanation for the observations in \cite{gil2024understanding} presented after Theorem \ref{the:3}, that is, the on-average variance \( \sigma \) can capture how the generalization gap changes with the fraction of random labels.

\begin{figure}[H] 
\begin{center}
\centerline{\includegraphics[width=7.1in]{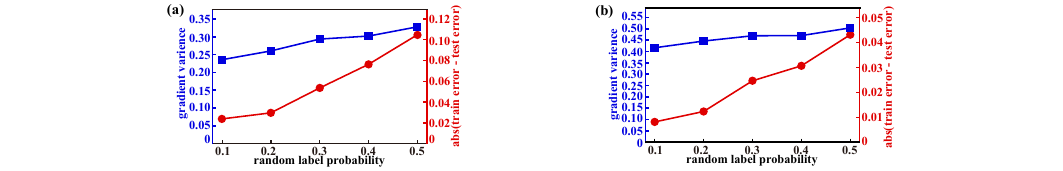}}
\caption{On-average variance and generalization gap for varying random label probabilities: (a) MNIST, (b) Fashion MNIST.}
\label{fig:4}
\end{center}
\end{figure}

\section{Conclusion}
In this paper, we study the generalization of QNNs through algorithmic stability. We first establish high-probability generalization bounds for QNNs via uniform stability. Our bounds provide practical insights into both the design and training processes for developing powerful QNNs. We next extend our analysis to the noise scenario, highlighting the potential benefits of quantum noise. We finally argue that previous results are coming from worst-case analysis and propose a refined optimization-dependent generalization bound. While our generalization bounds hold for arbitrary data distributions, an interesting direction is to explore the generalization of QNNs with a specific data distribution.

\section{Acknowledgements}
This work was partially supported by the National Natural Science Foundation of China (Grant No.~62102388), the Innovation Program for Quantum Science and Technology (Grant No.~2021ZD0302901), and the USTC Kunpeng \& Ascend Center of Excellence.

\bibliographystyle{unsrt}
\bibliography{ref}

\newpage
\appendix

\section{Notions}
The main notations of this paper are summarized in Table \ref{tab:Notions}.
\begin{table}[H]
\caption{Summary of main notations involved in this paper.}
\label{tab:Notions}
\begin{center}
\begin{tabular}{l|l}
\toprule
Notion & Description \\
\midrule
\( \mathcal{Z} \) & the sample space associated with input space \( \mathcal{X} \) and output space \( \mathcal{Y} \) \\
\( \mathcal{D} \) & the probability distribution defined on the sample space \( \mathcal{Z}\) \\
\( \boldsymbol{z}=(\boldsymbol{x},y) \) & the random example sampling from \( \mathcal{Z} \) \\
\( \boldsymbol{\theta} \) & the parameter of QNN\\
\( K_g \) & the number of quantum gates in QNN \\
\( K \) & the number of trainable quantum gates in QNN\\
\( O \) & the observable operator \\
\( S \) & the training dataset defined as \( S=\{\boldsymbol{z}_i=(\boldsymbol{x}_i, y_i):i=1,\ldots,m\} \) independently drawn from \( \mathcal{D} \) \\
\( m \) & the number of training examples \\
\( f_{\boldsymbol{\theta}}(\cdot) \) & the output function of QNN \\
\( \ell(\cdot) \) & the loss function of QNN \\
\( \nabla_{\boldsymbol{\theta}}\ell\) & the gradient of \( \ell(\cdot) \) to the argument \( \boldsymbol{\theta} \) \\
\( R_{\mathcal{D}},R_{S} \) & the population risk and empirical risk based on training dataset \( S \), respectively \\
\( T \) & the number of iterations for SGD\\
\( \boldsymbol{\theta}_t \) & the parameter  of QNN learned using the SGD algorithm for \( t \) iterations \\
\( \eta_t \) & the step size in iteration \( t \) \\
\( A,A(S) \) & the learning algorithm and its output model based on training dataset \( S \), respectively \\
\( \alpha_\ell,\nu_\ell \) & the parameters of Lipschitz continuity and smoothness, respectively \\
\bottomrule
\end{tabular}
\end{center}
\end{table}

\section{Auxiliary Lemmas}
In this section, we provide some auxiliary lemmas from quantum information theory that are essential for our main proofs.

To translate between the spectral norm of unitaries and the diamond norm of the corresponding channels, we employ the following lemma from \cite{caro2022generalization}.
\begin{lemma}[Spectral norm and diamond norm of unitary channels; see \cite{caro2022generalization}, Lemma 5]\label{lem:B.1}
    Let \( \mathcal{U}(\rho)=U\rho U^\dagger \) and \( \mathcal{V}(\rho)=V\rho V^\dagger \) be unitary channels. Then, \( \frac{1}{2}\|\mathcal{U}\left(|\psi\rangle\langle\psi|\right)-\mathcal{V}\left(|\psi\rangle\langle\psi|\right)\|_1\leq \|\left(U-V\right)|\psi\rangle\|_2 \) for any pure state \( |\psi\rangle \). Therefore,
    \begin{equation*}
        \frac{1}{2}\|\mathcal{U}-\mathcal{V}\|_\diamond\leq \|U-V\|.
    \end{equation*}
\end{lemma}
Moreover, we recall the following lemma from \cite{nielsen2001quantum}.
\begin{lemma}[see \cite{nielsen2001quantum}, section 4.5.3]\label{lem:B.2}
    Define the distance of two unitary matrices \( U_1,U_2 \) as the sepctral norm of the matrix \( U_1-U_2\), i.e., \( E(U_1,U_2)=\|U_1-U_2\| \). Then, 
    \begin{equation*}
        E(U_KU_{K-1}\ldots U_1,V_KV_{K-1}\ldots V_1)\leq \sum_{j=1}^KE(U_j,V_j) , 
    \end{equation*}
    where \( U_1,U_2,\ldots,U_K,V_1,V_2,\ldots,V_K \) are unitary matrices.
\end{lemma}
Next, we state and prove the following key lemma that provides a bound on the distance between two \( U(\boldsymbol{\theta}) \) with different parameter settings.
\begin{lemma}[Bound in Parameters Change]\label{lem:B.3}
    Suppose a parameterized unitary \( U(\boldsymbol{\theta}) = \prod_{k=1}^K V_k e^{-i \boldsymbol{\theta}_k P_k / 2} V_{K+1} \), we have the upper bound on two different parameter sets \( \boldsymbol{\theta}_1,\boldsymbol{\theta}_2\in\mathbb{R}^K\)
    \begin{equation*}
    \begin{aligned}
         \| U(\boldsymbol{\theta}_1) - U(\boldsymbol{\theta}_2) \| 
         &\leq \frac{\sqrt{K}}{2} \| \boldsymbol{\theta}_1 - \boldsymbol{\theta}_2 \|_2,
    \end{aligned}
    \end{equation*}
    where \(V_k\) are fixed quantum gates, \( P_k \in\{X, Y, Z\} \) denotes a single-qubit Pauli gate. For ease of readability, the tensor factors of identities accompanying the parametrized quantum gates \( e^{-i \boldsymbol{\theta}_k P_k / 2} \) are omitted. 
\end{lemma}
\begin{proof}
    Note that we can write \( U(\boldsymbol{\theta}_1) = V_1 e^{-i \theta_{1,1} P_1 / 2} V_{2} e^{-i \theta_{1,2} P_{2} / 2} \cdots V_{K} e^{-i \theta_{1,K} P_K / 2} V_{K+1} \), \( U(\boldsymbol{\theta}_2) = V_1 e^{-i \theta_{2,1} P_1 / 2} V_{2} e^{-i \theta_{2,2} P_{2} / 2} \cdots V_{K} e^{-i \theta_{2,K} P_K / 2} V_{K+1} \). By Lemma \ref{lem:B.2}, \( \| U(\boldsymbol{\theta}_1) - U(\boldsymbol{\theta}_2) \| \) can be bounded as
    \begin{equation} \label{eq:B.3.1}
    \begin{aligned}
        \| U(\boldsymbol{\theta}_1) - U(\boldsymbol{\theta}_2) \|
        &\leq \sum_{k=1}^{K+1} \| V_k - V_k \| + \sum_{k=1}^K \left\| e^{-i \frac{\theta_{1,k} P_k}{2}} - e^{-i \frac{\theta_{2,k} P_k}{2}} \right\| \\
        &= \sum_{k=1}^K \left\| e^{-i \frac{\theta_{1,k} P_k}{2}} - e^{-i \frac{\theta_{2,k} P_k}{2}} \right\|.
    \end{aligned}
    \end{equation}
    Additionally, the sub-term \( \| e^{-i \frac{\theta_{1,k} P_k}{2}} - e^{-i \frac{\theta_{2,k} P_k}{2}} \| \) can be written as
    \begin{equation*}
    \begin{aligned}
        \| e^{-i \frac{\theta_{1,k} P_k}{2}} - e^{-i \frac{\theta_{2,k} P_k}{2}} \| 
        &= \| I - e^{i \frac{(\theta_{1,k} - \theta_{2,k}) P_k}{2}} \| \nonumber \\
        &= \| I - \cos\left( \frac{\theta_{1,k} - \theta_{2,k}}{2} \right) I + i \sin\left( \frac{\theta_{1,k} - \theta_{2,k}}{2} \right) P_k \| \nonumber \\
        &= \sqrt{\left( 1 - \cos\left( \frac{\theta_{1,k} - \theta_{2,k}}{2} \right) \right)^2 + \left( \sin\left( \frac{\theta_{1,k} - \theta_{2,k}}{2} \right) \right)^2} \nonumber \\
        &= \left| 2 \sin\left( \frac{\theta_{1,k} - \theta_{2,k}}{4} \right) \right|.
    \end{aligned}
    \end{equation*}
    Plugging it into equation (\ref{eq:B.3.1}), we further get
    \begin{equation*}
    \begin{aligned}
        \| U(\boldsymbol{\theta}_1) - U(\boldsymbol{\theta}_2) \|
        &\leq \sum_{k=1}^K \left| 2 \sin\left( \frac{\theta_{1,k} - \theta_{2,k}}{4} \right) \right| \\
        &\leq \sum_{k=1}^K \left| \frac{\theta_{1,k} - \theta_{2,k}}{2} \right| \\
        &\leq \frac{\sqrt{K}}{2} \| \boldsymbol{\theta}_1 - \boldsymbol{\theta}_2 \|_2.
    \end{aligned}
    \end{equation*}
    This completes the proof of Lemma \ref{lem:B.3}.
\end{proof}
To examine the effects of depolarizing noise from the perspective of QNN generalization, we recall the following lemma from \cite{du2021learnability}.
\begin{lemma}[see \cite{du2021learnability}, Lemma 6]\label{lem:B.4} 
    Let \( \mathcal{E}_p \) be the depolarization channel. There always exists a depolarization channel \( \mathcal{E}_{\Tilde{p}} \) with \( \Tilde{p}=1-(1-p)^{K_g} \) that satisfies 
    \begin{equation*}
        \mathcal{E}_p\left\{ U_{K_g}(\boldsymbol{\theta}) \ldots U_2(\boldsymbol{\theta}) \mathcal{E}_p\left[ U_1(\boldsymbol{\theta}) \rho U_1^\dagger(\boldsymbol{\theta}) \right] U_2^\dagger(\boldsymbol{\theta}) \ldots U_{K_g}^\dagger(\boldsymbol{\theta}) \right\}
        = \mathcal{E}_{\Tilde{p}}\left[ U(\boldsymbol{\theta}) \rho U^\dagger(\boldsymbol{\theta}) \right],
    \end{equation*}
    where \( \rho \) is the input quantum state.
\end{lemma}

\section{Proofs of Main Results} \label{ap: Appendix C}
In this section, we provide the proofs of main results in our paper. We require several useful lemmas to prove the main results.
\begin{lemma}[From Loss Stability to Parameter Stability]\label{lem:C.1}
    Let \( \boldsymbol{\theta}_t \) and \( \boldsymbol{\theta}_t^\prime \) be the parameters of QNNs learned using the SGD algorithm for \( t \) iterations on training datasets \( S \) and \( S^\prime \), respectively. Then, the output difference of the QNNs is bounded by, 
    \begin{equation*}
    \begin{aligned}
        \left| f_{\boldsymbol{\theta}_t}(\boldsymbol{x}) - f_{\boldsymbol{\theta}_t^\prime}(\boldsymbol{x}) \right| 
        &\leq \sqrt{K} \| O \|\|\boldsymbol{\theta}_t-\boldsymbol{\theta}_t^\prime\|_2.
    \end{aligned}
    \end{equation*}
\end{lemma}
\begin{proof}
    The difference between the two output functions of the QNNs can be represented as follows
    \begin{equation*}
    \begin{aligned}
        \left| f_{\boldsymbol{\theta}_t}(\boldsymbol{x}) - f_{\boldsymbol{\theta}_t^\prime}(\boldsymbol{x}) \right| 
        &= \left| \mathrm{Tr}\left( O U(\boldsymbol{\theta}_t) \rho(\boldsymbol{x}) U^\dagger(\boldsymbol{\theta}_t) \right) 
        - \mathrm{Tr}\left( O U(\boldsymbol{\theta}_t^\prime) \rho(\boldsymbol{x}) U^\dagger (\boldsymbol{\theta}_t^\prime)\right) \right| \\
        &= \left| \mathrm{Tr}\left( O \left( U(\boldsymbol{\theta}_t) \rho(\boldsymbol{x}) U^\dagger (\boldsymbol{\theta}_t)
        - U(\boldsymbol{\theta}_t^\prime) \rho(\boldsymbol{x}) U^\dagger(\boldsymbol{\theta}_t^\prime) \right) \right) \right| \\
        &\leq \|O\| \left\| U(\boldsymbol{\theta}_t) \rho(\boldsymbol{x}) U^\dagger (\boldsymbol{\theta}_t)
        - U(\boldsymbol{\theta}_t^\prime) \rho(\boldsymbol{x}) U^\dagger(\boldsymbol{\theta}_t^\prime) \right\|_1,
    \end{aligned}
    \end{equation*}
    where the last inequality uses the Cauchy-Schwartz inequality.
    
    Let \( \mathcal{E}(\rho)=U(\boldsymbol{\theta}_t)\rho U^\dagger(\boldsymbol{\theta}_t) \), \( \mathcal{E}^\prime(\rho)=U(\boldsymbol{\theta}_t^\prime)\rho U^\dagger(\boldsymbol{\theta}_t^\prime) \) be unitary channels. By Lemma \ref{lem:B.1}, we have
    \begin{equation*}
    \begin{aligned}
        \left| f_{\boldsymbol{\theta}_t}(\boldsymbol{x}) - f_{\boldsymbol{\theta}_t^\prime}(\boldsymbol{x}) \right| 
        &\leq \|O\| \|\mathcal{E}(\rho(\boldsymbol{x})) - \mathcal{E}^\prime(\rho(\boldsymbol{x}))\|_1  \\
        &\leq \|O\| \|\mathcal{E} - \mathcal{E}^\prime \|_\diamond \\
        &\leq 2 \|O\| \| U(\boldsymbol{\theta}_t) - U(\boldsymbol{\theta}_t^\prime) \|.
    \end{aligned}
    \end{equation*} 
    Note that any \( U(\boldsymbol{\theta}) \) is of the form \( U(\boldsymbol{\theta})=V_1U_1V_2U_2V_3\ldots V_KU_KV_{K+1} \), where \(U_k, 1\leq k\leq K\), are a particular choice of the trainable single-qubit Pauli rotations and \(V_k,1\leq k\leq K+1\), are the non-trainable \( n \)-qubit unitaries. (For ease of readability, we have not written out the tensor factors of identities accompanying the \( U_k \).) We can write \( U(\boldsymbol{\theta}_t) = V_1 e^{-i \theta_{t,1} P_1 / 2} V_{2} e^{-i \theta_{t,2} P_{2} / 2} \cdots V_{K} e^{-i \theta_{t,K} P_K / 2} V_{K+1} \), \( U(\boldsymbol{\theta}_t^\prime) = V_1 e^{-i \theta_{t,1}^\prime P_1 / 2} V_{2} e^{-i \theta_{t,2}^\prime P_{2} / 2} \cdots V_{K} e^{-i \theta_{t,K}^\prime P_K / 2} V_{K+1} \), where \( P_k \in\{X, Y, Z\} \) denotes a single-qubit Pauli gate. By Lemma \ref{lem:B.3}, we further get
    \begin{equation} \label{eq:C.1.1}
    \begin{aligned}
         \left| f_{\boldsymbol{\theta}_t}(\boldsymbol{x}) - f_{\boldsymbol{\theta}_t^\prime}(\boldsymbol{x}) \right| 
         &\leq 2 \|O\| \sum_{k=1}^K \left| 2 \sin\left( \frac{\theta_{t,k} - \theta_{t,k}^\prime}{4} \right) \right| \\
         &\leq 2 \|O\| \sum_{k=1}^K \left| \frac{\theta_{t,k} - \theta_{t,k}^\prime}{2} \right| \\
         &\leq \sqrt{K} \| O \|\|\boldsymbol{\theta}_t-\boldsymbol{\theta}_t^\prime\|_2.
    \end{aligned}
    \end{equation}
    This completes the proof of Lemma \ref{lem:C.1}.
\end{proof}

\begin{lemma}[QNN Same Sample Loss Stability Bound] \label{lem:C.2}
     Suppose that Assumption \ref{ass:Lipschitzness} and \ref{ass:smoothness} hold. Let \( \boldsymbol{\theta}_t \) and \( \boldsymbol{\theta}_t^\prime \) be the parameters of two QNNs learned using the SGD algorithm for \( t \) iterations on two training datasets \( S \) and \( S^\prime \), respectively. Then, the loss derivative difference of the QNNs with respect to the same sample is bounded by,
    \begin{equation*}
        \|\nabla_{\boldsymbol{\theta}} \ell(f_{\boldsymbol{\theta}_t}(\boldsymbol{x}); y) - \nabla_{\boldsymbol{\theta}} \ell(f_{\boldsymbol{\theta}_t^\prime}(\boldsymbol{x}); y)\|_2 \leq\kappa\|\boldsymbol{\theta}_t-\boldsymbol{\theta}_t^\prime\|_2,
    \end{equation*}
    where \( \kappa=\alpha_\ell K \|O\| + \sqrt{2} \nu_\ell K \|O\|^2 \).
\end{lemma}
\begin{proof}
     Using the Assumption \ref{ass:Lipschitzness} and \ref{ass:smoothness} that the loss function is Lipschitz continuous and smoothness, we have
    \begin{equation}\label{eq:C.2.1}
    \begin{aligned}
        &\left\|\nabla_{\boldsymbol{\theta}} \ell\left(f_{\boldsymbol{\theta}_t}(\boldsymbol{x}); y\right) 
        - \nabla_{\boldsymbol{\theta}} \ell\left(f_{\boldsymbol{\theta}_t^\prime}(\boldsymbol{x}); y\right)\right\|_2 \\
        &= \left\| \frac{\partial \ell}{\partial f} \left(f_{\boldsymbol{\theta}_t}(\boldsymbol{x}); y\right) 
        \nabla_{\boldsymbol{\theta}} f_{\boldsymbol{\theta}_t}(\boldsymbol{x}) 
        - \frac{\partial \ell}{\partial f} \left(f_{\boldsymbol{\theta}_t^\prime}(\boldsymbol{x}); y\right) 
        \nabla_{\boldsymbol{\theta}} f_{\boldsymbol{\theta}_t^\prime}(\boldsymbol{x}) \right\|_2 \\
        &\leq \left\| \frac{\partial \ell}{\partial f} \left(f_{\boldsymbol{\theta}_t}(\boldsymbol{x}); y\right) 
        \left(\nabla_{\boldsymbol{\theta}} f_{\boldsymbol{\theta}_t}(\boldsymbol{x}) 
        - \nabla_{\boldsymbol{\theta}} f_{\boldsymbol{\theta}_t^\prime}(\boldsymbol{x})\right) \right\|_2 + \left\| \left( \frac{\partial \ell}{\partial f} \left(f_{\boldsymbol{\theta}_t}(\boldsymbol{x}); y\right) 
        - \frac{\partial \ell}{\partial f} \left(f_{\boldsymbol{\theta}_t^\prime}(\boldsymbol{x}); y\right)\right) 
        \nabla_{\boldsymbol{\theta}} f_{\boldsymbol{\theta}_t^\prime}(\boldsymbol{x}) \right\|_2 \\
        &= \left| \frac{\partial \ell}{\partial f} \left(f_{\boldsymbol{\theta}_t}(\boldsymbol{x}); y\right) \right| 
        \left\|\nabla_{\boldsymbol{\theta}} f_{\boldsymbol{\theta}_t}(\boldsymbol{x}) 
        - \nabla_{\boldsymbol{\theta}} f_{\boldsymbol{\theta}_t^\prime}(\boldsymbol{x})\right\|_2 + \left| \frac{\partial \ell}{\partial f} \left(f_{\boldsymbol{\theta}_t}(\boldsymbol{x}); y\right) 
        - \frac{\partial \ell}{\partial f} \left(f_{\boldsymbol{\theta}_t^\prime}(\boldsymbol{x}); y\right) \right| 
        \left\|\nabla_{\boldsymbol{\theta}} f_{\boldsymbol{\theta}_t^\prime}(\boldsymbol{x})\right\|_2 \\
        &\leq \alpha_\ell \left\|\nabla_{\boldsymbol{\theta}} f_{\boldsymbol{\theta}_t}(\boldsymbol{x}) 
        - \nabla_{\boldsymbol{\theta}} f_{\boldsymbol{\theta}_t^\prime}(\boldsymbol{x})\right\|_2 
        + \nu_\ell \left| f_{\boldsymbol{\theta}_t}(\boldsymbol{x}) - f_{\boldsymbol{\theta}_t^\prime}(\boldsymbol{x})\right| 
        \left\|\nabla_{\boldsymbol{\theta}} f_{\boldsymbol{\theta}_t^\prime}(\boldsymbol{x})\right\|_2.
    \end{aligned}
    \end{equation}
    The term \( \left\|\nabla_{\boldsymbol{\theta}} f_{\boldsymbol{\theta}_t}(\boldsymbol{x}) - \nabla_{\boldsymbol{\theta}} f_{\boldsymbol{\theta}_t^\prime}(\boldsymbol{x})\right\|_2 \) can be bounded by the individual terms \( \left|\nabla_{\theta_j} f_{\boldsymbol{\theta}_t}(\boldsymbol{x}) - \nabla_{\theta_j} f_{\boldsymbol{\theta}_t^\prime}(\boldsymbol{x})\right| \), which can be computed using parameter-shift rules \cite{mitarai2018quantum},
    \begin{equation}\label{eq:C.2.2}
    \begin{aligned}
        &\left|\nabla_{\theta_j} f_{\boldsymbol{\theta}_t}(\boldsymbol{x}) 
        - \nabla_{\theta_j} f_{\boldsymbol{\theta}_t^\prime}(\boldsymbol{x})\right| \\
        &= \frac{1}{2} \left| \left( f_{\boldsymbol{\theta}_t + \frac{\pi}{2} \boldsymbol{e}_j}(\boldsymbol{x}) 
        - f_{\boldsymbol{\theta}_t - \frac{\pi}{2} \boldsymbol{e}_j}(\boldsymbol{x}) \right) 
        - \left( f_{\boldsymbol{\theta}_t^\prime + \frac{\pi}{2} \boldsymbol{e}_j}(\boldsymbol{x}) 
        - f_{\boldsymbol{\theta}_t^\prime - \frac{\pi}{2} \boldsymbol{e}_j}(\boldsymbol{x}) \right) \right| \\
        &= \frac{1}{2} \left| \left( f_{\boldsymbol{\theta}_t + \frac{\pi}{2} \boldsymbol{e}_j}(\boldsymbol{x}) 
        - f_{\boldsymbol{\theta}_t^\prime + \frac{\pi}{2} \boldsymbol{e}_j}(\boldsymbol{x}) \right) 
        - \left( f_{\boldsymbol{\theta}_t - \frac{\pi}{2} \boldsymbol{e}_j}(\boldsymbol{x}) 
        - f_{\boldsymbol{\theta}_t^\prime - \frac{\pi}{2} \boldsymbol{e}_j}(\boldsymbol{x}) \right) \right| \\
        &\leq \frac{1}{2} \Big[ \left| f_{\boldsymbol{\theta}_t + \frac{\pi}{2} \boldsymbol{e}_j}(\boldsymbol{x}) 
        - f_{\boldsymbol{\theta}_t^\prime + \frac{\pi}{2} \boldsymbol{e}_j}(\boldsymbol{x}) \right| + \left| f_{\boldsymbol{\theta}_t - \frac{\pi}{2} \boldsymbol{e}_j}(\boldsymbol{x}) 
        - f_{\boldsymbol{\theta}_t^\prime - \frac{\pi}{2} \boldsymbol{e}_j}(\boldsymbol{x}) \right| \Big],
    \end{aligned}
    \end{equation}
    where \( \boldsymbol{e}_j \) is the unit vector along the \( \boldsymbol{\theta}_j \) axis.
    
    According to Lemma \ref{lem:C.1}, we have
    \begin{equation*}
        \left| f_{\boldsymbol{\theta}_t + \frac{\pi}{2} \boldsymbol{e}_j}(\boldsymbol{x}) 
        - f_{\boldsymbol{\theta}_t^\prime + \frac{\pi}{2} \boldsymbol{e}_j}(\boldsymbol{x}) \right| \leq \sqrt{K} \|O\| \|\boldsymbol{\theta}_t - \boldsymbol{\theta}_t^\prime\|_2,
    \end{equation*}
    and
    \begin{equation*}
        \left| f_{\boldsymbol{\theta}_t - \frac{\pi}{2} \boldsymbol{e}_j}(\boldsymbol{x}) 
        - f_{\boldsymbol{\theta}_t^\prime - \frac{\pi}{2} \boldsymbol{e}_j}(\boldsymbol{x}) \right| \leq \sqrt{K} \|O\| \|\boldsymbol{\theta}_t - \boldsymbol{\theta}_t^\prime\|_2.
    \end{equation*}
    Plugging it into equation (\ref{eq:C.2.2}), we further get
    \begin{equation*} 
        \left|\nabla_{\theta_j} f_{\boldsymbol{\theta}_t}(\boldsymbol{x}) 
        - \nabla_{\theta_j} f_{\boldsymbol{\theta}_t^\prime}(\boldsymbol{x})\right| \leq \sqrt{K} \|O\| \|\boldsymbol{\theta}_t - \boldsymbol{\theta}_t^\prime\|_2.
    \end{equation*}
    Therefore, we can bound the term \( \|\nabla_{\boldsymbol{\theta}} f_{\boldsymbol{\theta}_t}(\boldsymbol{x}) - \nabla_{\boldsymbol{\theta}} f_{\boldsymbol{\theta}_t^\prime}(\boldsymbol{x})\|_2 \) as follows
    \begin{equation}\label{eq:C.2.3}
    \begin{aligned}
        \|\nabla_{\boldsymbol{\theta}} f_{\boldsymbol{\theta}_t}(\boldsymbol{x}) 
        - \nabla_{\boldsymbol{\theta}} f_{\boldsymbol{\theta}_t^\prime}(\boldsymbol{x})\|_2
        &= \sqrt{\sum_{j=1}^K \left| \nabla_{\theta_j} f_{\boldsymbol{\theta}_t}(\boldsymbol{x}) 
        - \nabla_{\theta_j} f_{\boldsymbol{\theta}_t^\prime}(\boldsymbol{x}) \right|^2} \\
        &\leq \sqrt{K \left( \sqrt{K} \|O\| \|\boldsymbol{\theta}_t - \boldsymbol{\theta}_t^\prime\|_2 \right)^2} \\
        &= K \|O\| \|\boldsymbol{\theta}_t - \boldsymbol{\theta}_t^\prime\|_2.
    \end{aligned}
    \end{equation}
    Similarly, the term \( \left\|\nabla_{\boldsymbol{\theta}} f_{\boldsymbol{\theta}_t^\prime}(\boldsymbol{x})\right\|_2 \) can be bounded by the individual terms  \( \left\|\nabla_{\theta_j} f_{\boldsymbol{\theta}_t^\prime}(\boldsymbol{x})\right\|_2 \). Additionally, by equation (\ref{eq:C.1.1}) in Lemma \ref{lem:C.1}, we have 
    \begin{equation}\label{eq:C.2.4}
    \begin{aligned}
        \left| \nabla_{\theta_j} f_{\boldsymbol{\theta}_t^\prime}(\boldsymbol{x}) \right| 
        &= \frac{1}{2} \left| f_{\boldsymbol{\theta}_t^\prime + \frac{\pi}{2} \boldsymbol{e}_j}(\boldsymbol{x}) 
        - f_{\boldsymbol{\theta}_t^\prime - \frac{\pi}{2} \boldsymbol{e}_j}(\boldsymbol{x}) \right| \\
        &\leq \|O\| \left( \sum_{k \neq j} \left| 2 \sin\left( \frac{\theta_{t,k}^\prime - \theta_{t,k}^\prime}{4} \right) \right| 
        + \left| 2 \sin\left( \frac{(\theta_{t,j}^\prime + \frac{\pi}{2}) - (\theta_{t,j}^\prime - \frac{\pi}{2})}{4} \right) \right| \right) \\
        &= \sqrt{2} \|O\|.
    \end{aligned}
    \end{equation}
    Therefore, we can bound  the term \( \left\|\nabla_{\boldsymbol{\theta}} f_{\boldsymbol{\theta}_t^\prime}(\boldsymbol{x})\right\|_2 \) as follows
    \begin{equation}\label{eq:C.2.5}
    \begin{aligned}
        \|\nabla_{\boldsymbol{\theta}} f_{\boldsymbol{\theta}_t^\prime}(\boldsymbol{x})\|_2
        &= \sqrt{\sum_{j=1}^K \left|\nabla_{\theta_j} f_{\boldsymbol{\theta}_t^\prime}(\boldsymbol{x})\right|^2} \\
        &\leq \sqrt{K \left(\sqrt{2} \|O\|\right)^2} \\
        &= \sqrt{2K} \|O\|.
    \end{aligned}
    \end{equation}
    Finally, according to Lemma \ref{lem:C.1}, the term \(  \left| f_{\boldsymbol{\theta}_t}(\boldsymbol{x}) - f_{\boldsymbol{\theta}_t^\prime}(\boldsymbol{x})\right| \) can be bounded as
    \begin{equation} \label{eq:C.2.6}
        \left| f_{\boldsymbol{\theta}_t}(\boldsymbol{x}) - f_{\boldsymbol{\theta}_t^\prime}(\boldsymbol{x})\right| \leq \sqrt{K} \|O\| \|\boldsymbol{\theta}_t - \boldsymbol{\theta}_t^\prime\|_2
    \end{equation}
    Plugging equation (\ref{eq:C.2.3}), (\ref{eq:C.2.5}), and (\ref{eq:C.2.6}) back into equation (\ref{eq:C.2.1}) completes the proof of Lemma \ref{lem:C.2}.
\end{proof}

\begin{lemma}[QNN Differnet Sample Loss Stability Bound] \label{lem:C.3}
     Suppose that Assumption \ref{ass:Lipschitzness} hold. Let \( \boldsymbol{\theta}_t \) and \( \boldsymbol{\theta}_t^\prime \) be the parameters of two QNNs learned using the SGD algorithm for \( t \) iterations on two training datasets \( S \) and \( S^\prime \), respectively. Then, the loss derivative difference of the QNNs with respect to the different sample is bounded by,
    \begin{equation*}
        \|\nabla_{\boldsymbol{\theta}} \ell(f_{\boldsymbol{\theta}_t}(\boldsymbol{x}); y) - \nabla_{\boldsymbol{\theta}} \ell(f_{\boldsymbol{\theta}_t^\prime}(\boldsymbol{x}^\prime); y^\prime)\|_2 \leq2\sqrt{2}\alpha_\ell\sqrt{K}\|O\|.
    \end{equation*}
\end{lemma}
\begin{proof}
     Using the Assumption \ref{ass:Lipschitzness} that the loss function is Lipschitz continuous, we have
    \begin{equation}\label{eq:C.3.1}
    \begin{aligned}
        &\left\|\nabla_{\boldsymbol{\theta}} \ell\left(f_{\boldsymbol{\theta}_t}(\boldsymbol{x}); y\right) 
        - \nabla_{\boldsymbol{\theta}} \ell\left(f_{\boldsymbol{\theta}_t^\prime}(\boldsymbol{x}^\prime); y^\prime\right)\right\|_2 \\
        &= \left\| \frac{\partial \ell}{\partial f} \left(f_{\boldsymbol{\theta}_t}(\boldsymbol{x}); y\right) 
        \nabla_{\boldsymbol{\theta}} f_{\boldsymbol{\theta}_t}(\boldsymbol{x}) 
        - \frac{\partial \ell}{\partial f} \left(f_{\boldsymbol{\theta}_t^\prime}(\boldsymbol{x}^\prime); y^\prime\right) 
        \nabla_{\boldsymbol{\theta}} f_{\boldsymbol{\theta}_t^\prime}(\boldsymbol{x}^\prime) \right\|_2 \\
        &\leq \left\| \frac{\partial \ell}{\partial f} \left(f_{\boldsymbol{\theta}_t}(\boldsymbol{x}); y\right) 
        \nabla_{\boldsymbol{\theta}} f_{\boldsymbol{\theta}_t}(\boldsymbol{x}) \right\|_2 
        + \left\| \frac{\partial \ell}{\partial f} \left(f_{\boldsymbol{\theta}_t^\prime}(\boldsymbol{x}^\prime); y^\prime\right) 
        \nabla_{\boldsymbol{\theta}} f_{\boldsymbol{\theta}_t^\prime}(\boldsymbol{x}^\prime) \right\|_2 \\
        &= \left| \frac{\partial \ell}{\partial f} \left(f_{\boldsymbol{\theta}_t}(\boldsymbol{x}); y\right) \right|  
        \left\|\nabla_{\boldsymbol{\theta}} f_{\boldsymbol{\theta}_t}(\boldsymbol{x})\right\|_2 
        + \left| \frac{\partial \ell}{\partial f} \left(f_{\boldsymbol{\theta}_t^\prime}(\boldsymbol{x}^\prime); y^\prime\right) \right|  
        \left\|\nabla_{\boldsymbol{\theta}} f_{\boldsymbol{\theta}_t^\prime}(\boldsymbol{x}^\prime)\right\|_2 \\
        &\leq \alpha_\ell \left(
        \left\|\nabla_{\boldsymbol{\theta}} f_{\boldsymbol{\theta}_t}(\boldsymbol{x})\right\|_2 
        + \left\|\nabla_{\boldsymbol{\theta}} f_{\boldsymbol{\theta}_t^\prime}(\boldsymbol{x}^\prime)\right\|_2
        \right).
    \end{aligned}
    \end{equation}
    By equation (\ref{eq:C.2.5}) in Lemma \ref{lem:C.2}, we similarly bound \( \left\|\nabla_{\boldsymbol{\theta}} f_{\boldsymbol{\theta}_t}(\boldsymbol{x})\right\|_2 \leq \sqrt{2K}\|O\|, \left\|\nabla_{\boldsymbol{\theta}} f_{\boldsymbol{\theta}_t^\prime}(\boldsymbol{x}^\prime)\right\|_2 \leq \sqrt{2K}\|O\| \). Plugging this into equation (\ref{eq:C.3.1}) completes the proof of Lemma \ref{lem:C.3}.  
\end{proof}
\subsection{Proof of Theorem \ref{the:1}}
In this subsection, we prove Theorem \ref{the:1} on the uniform stability of QNNs. We begin by quoting the result which we set to prove.
\renewcommand{\thetheorem}{\ref{the:1}}
\begin{theorem}[Uniform Stability Bound]
    Suppose that Assumption \ref{ass:Lipschitzness} and \ref{ass:smoothness} hold. Let \( A(S) \) be the QNN model trained on the dataset \( S\in\mathcal{Z}^m \) using the SGD algorithm with step sizes \( \eta_t\) for \( T \) iterations, then \( A(S) \) is \( \epsilon \)-uniformly stable with 
    \begin{equation*} 
        \epsilon\leq\sum_{t=0}^{T-1}\left[\prod_{j=t+1}^{T-1}(1+\eta_j\kappa)\right]\frac{2\sqrt{2}\eta_t\alpha_\ell^2K\|O\|^2}{m},
    \end{equation*}
    where
    \begin{equation*}
         \kappa:=\alpha_\ell K \|O\| + \sqrt{2} \nu_\ell K \|O\|^2.
    \end{equation*}
\end{theorem}
\begin{proof}
    Let \( S \) and \( S^\prime\) be two datasets of size \( m \) differing in only a single sample. Consider two sequences of the parameters, \( \{\boldsymbol{\theta}_0,\boldsymbol{\theta}_1,\ldots,\boldsymbol{\theta}_T\}\) and \( \{\boldsymbol{\theta}_0^\prime,\boldsymbol{\theta}_1^\prime,\ldots,\boldsymbol{\theta}_T^\prime \} \), learned by the QNN running SGD on \( S \) and \( S^\prime \), respectively. Let \( \delta_t=\|\boldsymbol{\theta}_t-\boldsymbol{\theta}_t^\prime\|_2 \).
     
    Using the Assumption \ref{ass:Lipschitzness} that the loss function is Lipschitz continuous, the linearity of expectation and Lemma \ref{lem:C.1}, we have
    \begin{equation}\label{eq:4.1.1}
    \begin{aligned}
        \left| \mathbb{E}_A \left[ \ell (f_{\boldsymbol{\theta}_T}(\boldsymbol{x}); y) - \ell (f_{\boldsymbol{\theta}_T^\prime}(\boldsymbol{x}); y) \right] \right|
        &\leq \alpha_\ell \mathbb{E}_A \left[ \left| f_{\boldsymbol{\theta}_T}(\boldsymbol{x}) - f_{\boldsymbol{\theta}_T^\prime}(\boldsymbol{x}) \right| \right] \\
        &\leq \alpha_\ell \sqrt{K} \|O\| \mathbb{E}_A\left[ \|\boldsymbol{\theta}_T - \boldsymbol{\theta}_T^\prime\|_2 \right] \\
        &\leq \alpha_\ell \sqrt{K} \|O\| \mathbb{E}_A\left[ \delta_T \right].
    \end{aligned}
    \end{equation}
    Then, we focus on the term \( \mathbb{E}_A\left[ \delta_T \right] \). Observe that at iteration \( t \), with probability \( 1-1/m \), the example selected by SGD is the same in both \( S \) and \( S^\prime \). With probability \( 1/m \) the selected example is different. Therefore, we have
    \begin{equation}\label{eq:4.1.2}
    \begin{aligned}
    \mathbb{E}_A[\delta_{t+1}] 
    &\leq (1 - \frac{1}{m}) \mathbb{E}_A \left[ \left\| \left( \boldsymbol{\theta}_{t} - \eta_t \nabla_{\boldsymbol{\theta}} \ell(f_{\boldsymbol{\theta}_t}(\boldsymbol{x}); y) \right) - \left( \boldsymbol{\theta}_{t}^\prime - \eta_t \nabla_{\boldsymbol{\theta}} \ell(f_{\boldsymbol{\theta}_t^\prime}(\boldsymbol{x}); y) \right) \right\|_2 \right] \\
    & \quad + \frac{1}{m} \mathbb{E}_A \left[ \left\| \left( \boldsymbol{\theta}_{t} - \eta_t \nabla_{\boldsymbol{\theta}} \ell(f_{\boldsymbol{\theta}_t}(\boldsymbol{x}^\prime); y^\prime) \right) - \left( \boldsymbol{\theta}_{t}^\prime - \eta_t \nabla_{\boldsymbol{\theta}} \ell(f_{\boldsymbol{\theta}_t^\prime}(\boldsymbol{x}^{\prime\prime}); y^{\prime\prime}) \right) \right\|_2 \right] \\
    &= \mathbb{E}_A[\delta_{t}] + (1 - \frac{1}{m}) \eta_t \mathbb{E}_A \left[ \left\| \nabla_{\boldsymbol{\theta}} \ell(f_{\boldsymbol{\theta}_t}(\boldsymbol{x}); y) - \nabla_{\boldsymbol{\theta}} \ell(f_{\boldsymbol{\theta}_t^\prime}(\boldsymbol{x}); y) \right\|_2 \right] \\
    & \quad + \frac{1}{m} \eta_t \mathbb{E}_A \left[ \left\| \nabla_{\boldsymbol{\theta}} \ell(f_{\boldsymbol{\theta}_t}(\boldsymbol{x}^\prime); y^\prime) - \nabla_{\boldsymbol{\theta}} \ell(f_{\boldsymbol{\theta}_t^\prime}(\boldsymbol{x}^{\prime\prime}); y^{\prime\prime}) \right\|_2 \right].
    \end{aligned}
    \end{equation}
    According to Lemma \ref{lem:C.2} and Lemma \ref{lem:C.3}, we further get
    \begin{equation*}
    \begin{aligned}
        \mathbb{E}_A[\delta_{t+1}] 
        &\leq \mathbb{E}_A[\delta_{t}] + (1 - \frac{1}{m}) \eta_t \cdot \mathbb{E}_A[\kappa \delta_t] + \frac{1}{m} \eta_t \cdot \mathbb{E}_A[2\sqrt{2} \alpha_\ell \sqrt{K} \|O\|] \\
        &\leq (1 + \eta_t \kappa) \mathbb{E}_A[\delta_t] + \frac{2\sqrt{2} \eta_t \alpha_\ell \sqrt{K} \|O\|}{m}.
    \end{aligned}
    \end{equation*}
    Unraveling the recursion gives
    \begin{equation*}
        \mathbb{E}_A[\delta_T] \leq \sum_{t=0}^{T-1} \left[ \prod_{j=t+1}^{T-1} \left( 1 + \eta_j \kappa \right) \right] \frac{2\sqrt{2} \eta_t \alpha_\ell \sqrt{K} \|O\|}{m}.
    \end{equation*}
    Plugging it into equation (\ref{eq:4.1.1}), we obtain
    \begin{equation*}
         \left| \mathbb{E}_A \left[ \ell (f_{\boldsymbol{\theta}_T}(\boldsymbol{x}); y) - \ell (f_{\boldsymbol{\theta}_T^\prime}(\boldsymbol{x}); y) \right] \right| \leq \sum_{t=0}^{T-1} \left[ \prod_{j=t+1}^{T-1} \left( 1 + \eta_j \kappa \right) \right] \frac{2\sqrt{2} \eta_t \alpha_\ell^2 K \|O\|^2}{m}.
    \end{equation*}
    By the definition of uniform stability as shown in Definition \ref{def:Uniform Stability}, we obtain the desired bound on the uniform stability of QNNs. This completes the proof of Theorem \ref{the:1}.
\end{proof}

\subsection{Proof of Corollary \ref{cor:1}}
In this subsection, we prove Corollary \ref{cor:1}, which provides simplified uniform stability results for two commonly used step sizes of QNNs. We first introduce an extension of Lemma 3.11 in \cite{hardt2016train}, specifically adapted to the unique context of quantum neural networks. The following lemma is motivated by the fact that SGD typically runs several iterations before encountering the different example between \( S \) and \( S^\prime \).
\renewcommand{\thetheorem}{C.4}
\begin{lemma} \label{lem:C.4}
    Suppose that Assumption \ref{ass:Lipschitzness} and \ref{ass:smoothness} hold, and \( \ell(\cdot,\cdot)\in[0,M]\). Let \( S \) and \( S^\prime \) of size \( m \) differing in only a single example. Consider two sequences of parameters, \( \{\boldsymbol{\theta}_0,\boldsymbol{\theta}_1,\ldots,\boldsymbol{\theta}_T\}\) and \( \{\boldsymbol{\theta}_0^\prime,\boldsymbol{\theta}_1^\prime,\ldots,\boldsymbol{\theta}_T^\prime \} \), learned by the QNN running SGD on \( S \) and \( S^\prime \), respectively. Let \( \delta_t=\|\boldsymbol{\theta}_t-\boldsymbol{\theta}_t^\prime\|_2 \). Then, for any \( \boldsymbol{z}\in\mathcal{Z} \) and \( t_0\in\{0,1,\ldots,m\} \), we have
    \begin{equation*}
    \begin{aligned}
        \left| \mathbb{E}_A[\ell(f_{\boldsymbol{\theta}_T}(\boldsymbol{x}); y) - \ell(f_{\boldsymbol{\theta}_T^\prime}(\boldsymbol{x}); y)] \right|
        &\leq \alpha_\ell \sqrt{K} \|O\| \mathbb{E}_A[\delta_T \mid \delta_{t_0}=0] + \frac{t_0 M}{m}.
    \end{aligned}
    \end{equation*}
\end{lemma}
\begin{proof}
    Let \( \mathcal{E} \) denote the event that \( \delta_{t_0}=0 \). Then we have
    \begin{equation} \label{eq:C.4.1}
    \begin{aligned}
        &|\mathbb{E}[\ell(f_{\boldsymbol{\theta}_T}(\boldsymbol{x}); y) - \ell(f_{\boldsymbol{\theta}_T^\prime}(\boldsymbol{x}); y)]| \\
        &\leq \mathbb{E}[|\ell(f_{\boldsymbol{\theta}_T}(\boldsymbol{x}); y) - \ell(f_{\boldsymbol{\theta}_T^\prime}(\boldsymbol{x}); y)|] \\
        &= \mathbb{P}[\mathcal{E}]\mathbb{E}\left[ \left|\ell(f_{\boldsymbol{\theta}_T}(\boldsymbol{x}); y) - \ell(f_{\boldsymbol{\theta}_T^\prime}(\boldsymbol{x}); y)\right| \mid \mathcal{E} \right] \\
        &\quad + \mathbb{P}[\mathcal{E}^c] \cdot \mathbb{E}\left[ \left|\ell(f_{\boldsymbol{\theta}_T}(\boldsymbol{x}); y) - \ell(f_{\boldsymbol{\theta}_T^\prime}(\boldsymbol{x}); y) \right| \mid \mathcal{E}^c \right] \\
        &\leq \mathbb{E}\left[ \left| \ell(f_{\boldsymbol{\theta}_T}(\boldsymbol{x}); y) - \ell(f_{\boldsymbol{\theta}_T^\prime}(\boldsymbol{x}); y) \right| \mid \mathcal{E} \right] \\
        &\quad + \mathbb{P}[\mathcal{E}^c] \cdot \sup \left|\ell(f_{\boldsymbol{\theta}_T}(\boldsymbol{x}); y) - \ell(f_{\boldsymbol{\theta}_T^\prime}(\boldsymbol{x}); y)\right|.
    \end{aligned}
    \end{equation}
    Using the Assumption \ref{ass:Lipschitzness} that the loss function is Lipschitz continuous and Lemma \ref{lem:C.1}, the first term \(  \mathbb{E}\left[ \left| \ell(f_{\boldsymbol{\theta}_T}(\boldsymbol{x}); y) - \ell(f_{\boldsymbol{\theta}_T^\prime}(\boldsymbol{x}); y) \right| \mid \mathcal{E} \right] \) can be bounded as
    \begin{equation} \label{eq:C.4.2}
        \mathbb{E}\left[ \left| \ell(f_{\boldsymbol{\theta}_T}(\boldsymbol{x}); y) - \ell(f_{\boldsymbol{\theta}_T^\prime}(\boldsymbol{x}); y) \right| \mid \mathcal{E} \right] \leq \alpha_\ell \sqrt{K} \|O\| \mathbb{E}[\delta_T \mid \mathcal{E}].
    \end{equation}
    It remains to bound the second term \( \mathbb{P}[\mathcal{E}^c] \cdot \sup \left|\ell(f_{\boldsymbol{\theta}_T}(\boldsymbol{x}); y) - \ell(f_{\boldsymbol{\theta}_T^\prime}(\boldsymbol{x}); y)\right| \). Let \(i^*\) be the position where \( S \) and \( S^\prime \) are different and denote the first time SGD uses the example \( \boldsymbol{z}_{i^*}\) by the random variable \( I \). Note that when \( I > t_0 \), then we must have that \( \delta_{t_0} = 0 \), since the execution on \( S \) and \( S^\prime \) is identical until iteration \( t_0 \). We have that
    \begin{equation*}
        \mathbb{P}[\mathcal{E}^c] = \mathbb{P}[\delta_{t_0} \neq 0] \leq \mathbb{P}[I \leq t_0] \leq \frac{t_0}{m}
    \end{equation*}
    According the condition \( \ell(\cdot,\cdot)\in[0,M] \), the second term can be bounded as
    \begin{equation} \label{eq:C.4.3}
        \mathbb{P}[\mathcal{E}^c] \cdot \sup \left|\ell(f_{\boldsymbol{\theta}_T}(\boldsymbol{x}); y) - \ell(f_{\boldsymbol{\theta}_T^\prime}(\boldsymbol{x}); y)\right| \leq \frac{t_0M}{m}.
    \end{equation}
    Plugging equation (\ref{eq:C.4.2}) and (\ref{eq:C.4.3}) back into equation (\ref{eq:C.4.1}) completes the proof of Lemma \ref{lem:C.4}
\end{proof}
We are now ready to prove Corollary \ref{cor:1}.
\renewcommand{\thetheorem}{\ref{cor:1}}
\begin{corollary}
    Suppose that Assumption \ref{ass:Lipschitzness} and \ref{ass:smoothness} hold, and \( \ell(\cdot,\cdot)\in[0,M] \). Let \( A(S) \) be the QNN model trained on the dataset \( S\in\mathcal{Z}^m \) using the SGD algorithm with step sizes \( \eta_t\) for \( T \) iterations. 
    \begin{enumerate}
        \item [(a)] If we choose the constant step sizes \( \eta_t=\eta \), then \( A(S) \) is \( \epsilon \)-uniformly stable with
        \begin{equation*}
            \epsilon\leq\frac{2\sqrt{2}\alpha_\ell^2K\|O\|^2}{\kappa m}(1+\eta\kappa)^T.
        \end{equation*}
        \item [(b)] If we choose the monotonically non-increasing step sizes \( \eta_t\leq c/(t+1) \), \( c>0 \), then \( A(S) \) is \( \epsilon \)-uniformly stable with
        \begin{equation*}
            \epsilon \leq \frac{1 + 1/c\kappa}{m}M^{\frac{c\kappa}{c\kappa + 1}} 
            \left( 2\sqrt{2}c\alpha_\ell^2 K \|O\|^2 \right)^{\frac{1}{c\kappa + 1}}
            T^{\frac{c\kappa}{c\kappa + 1}},
        \end{equation*}       
    \end{enumerate}
    where \( \kappa \) is defined by equation (\ref{eq:k}).
\end{corollary}
\begin{proof}
    Let \( S \) and \( S^\prime\) be two datasets of size \( m \) differing in only a single sample. Consider two sequences of the parameters, \( \{\boldsymbol{\theta}_0,\boldsymbol{\theta}_1,\ldots,\boldsymbol{\theta}_T\}\) and \( \{\boldsymbol{\theta}_0^\prime,\boldsymbol{\theta}_1^\prime,\ldots,\boldsymbol{\theta}_T^\prime \} \), learned by the QNN running SGD on \( S \) and \( S^\prime \), respectively. Let \( \delta_t=\|\boldsymbol{\theta}_t-\boldsymbol{\theta}_t^\prime\|_2 \). 
    
    For the constant step sizes \( \eta_t=\eta \), by Theorem \ref{the:1}, we have
    \begin{equation*}
    \begin{aligned}
         \epsilon
         &\leq \sum_{t=0}^{T-1}\left[\prod_{j=t+1}^{T-1}(1+\eta\kappa)\right]\frac{2\sqrt{2}\eta\alpha_\ell^2K\|O\|^2}{m} \\
         &\leq \frac{2\sqrt{2}\eta\alpha_\ell^2K\|O\|^2}{m}\sum_{t=0}^{T-1}(1+\eta\kappa)^t \\
         &\leq \frac{2\sqrt{2}\eta\alpha_\ell^2K\|O\|^2}{m} \cdot \frac{(1+\eta\kappa)^T - 1}{\eta\kappa}  \\
         &\leq \frac{2\sqrt{2}\alpha_\ell^2K\|O\|^2}{\kappa m}(1+\eta\kappa)^T.
    \end{aligned}
    \end{equation*}
    We immediately get the claimed upper bound on the uniform stability under the constant step size setting. This completes the proof of Corollary \ref{cor:1}, Part(a).
    
    For the decaying step sizes \( \eta_t\leq c/(t+1) \), by Lemma \ref{lem:C.4}, we have for every \( t_0\in\{0,1,\ldots,m\} \),
    \begin{equation} \label{eq:4.3.1}
    \begin{aligned}
        \left| \mathbb{E}_A[\ell(f_{\boldsymbol{\theta}_T}(\boldsymbol{x}); y) - \ell(f_{\boldsymbol{\theta}_T^\prime}(\boldsymbol{x}); y)] \right|
        &\leq \alpha_\ell \sqrt{K} \|O\| \mathbb{E}_A[\delta_T \mid \delta_{t_0}=0] + \frac{t_0 M}{m}.
    \end{aligned}
    \end{equation}
    Let \( \Delta_t = \mathbb{E}[\delta_t \mid \delta_{t_0} = 0] \). Observe that at iteration \( t \), with probability \( 1-1/m \), the example selected by SGD is the same in both \( S \) and \( S^\prime \). With probability \( 1/m \) the selected example is different. Therefore, we have
    \begin{equation} \label{eq:4.3.2}
    \begin{aligned}
    \Delta_{t+1} 
    &\leq (1 - \frac{1}{m}) \mathbb{E}_A \left[ \left\| \left( \boldsymbol{\theta}_{t} - \eta_t \nabla_{\boldsymbol{\theta}} \ell(f_{\boldsymbol{\theta}_t}(\boldsymbol{x}); y) \right)- \left( \boldsymbol{\theta}_{t}^\prime - \eta_t \nabla_{\boldsymbol{\theta}} \ell(f_{\boldsymbol{\theta}_t^\prime}(\boldsymbol{x}); y) \right) \right\|_2 \mid \delta_{t_0}=0 \right] \\
    & \quad + \frac{1}{m} \mathbb{E}_A \left[ \left\| \left( \boldsymbol{\theta}_{t} - \eta_t \nabla_{\boldsymbol{\theta}} \ell(f_{\boldsymbol{\theta}_t}(\boldsymbol{x}^\prime); y^\prime) \right) - \left( \boldsymbol{\theta}_{t}^\prime - \eta_t \nabla_{\boldsymbol{\theta}} \ell(f_{\boldsymbol{\theta}_t^\prime}(\boldsymbol{x}^{\prime\prime}); y^{\prime\prime}) \right) \right\|_2 \mid \delta_{t_0}=0 \right] \\
    &= \Delta_t + (1 - \frac{1}{m}) \eta_t \mathbb{E}_A \left[ \left\| \nabla_{\boldsymbol{\theta}} \ell(f_{\boldsymbol{\theta}_t}(\boldsymbol{x}); y) - \nabla_{\boldsymbol{\theta}} \ell(f_{\boldsymbol{\theta}_t^\prime}(\boldsymbol{x}); y) \right\|_2 \mid \delta_{t_0}=0 \right] \\
    & \quad + \frac{1}{m} \eta_t \mathbb{E}_A \left[ \left\| \nabla_{\boldsymbol{\theta}} \ell(f_{\boldsymbol{\theta}_t}(\boldsymbol{x}^\prime); y^\prime) - \nabla_{\boldsymbol{\theta}} \ell(f_{\boldsymbol{\theta}_t^\prime}(\boldsymbol{x}^{\prime\prime}); y^{\prime\prime}) \right\|_2 \mid \delta_{t_0}=0 \right].
    \end{aligned}
    \end{equation}
    According to Lemma \ref{lem:C.2} and Lemma \ref{lem:C.3}, we further get
   \begin{equation*}
    \begin{aligned}
        \Delta_{t+1}  
        &\leq \Delta_{t} 
        + \left(1 - \frac{1}{m}\right) \eta_t \cdot \mathbb{E}_A \big[\kappa \delta_t \;\big|\; \delta_{t_0} = 0 \big] 
        + \frac{1}{m} \eta_t \cdot \mathbb{E}_A \big[2\sqrt{2} \alpha_\ell \sqrt{K} \|O\| \;\big|\; \delta_{t_0} = 0 \big] \\
        &\leq (1 + \eta_t \kappa) \Delta_{t} 
        + \frac{2\sqrt{2} \eta_t \alpha_\ell \sqrt{K} \|O\|}{m}.
    \end{aligned}
    \end{equation*}
    Using the fact that \( \Delta_{t_0} = 0 \), we can unwind this recurrence relation from \( T \) down to \(t_0 + 1\). This gives
   \begin{equation*}
    \begin{aligned}
        \Delta_{t+1}
        &= \sum_{t = t_0 + 1}^T 
        \left[ \prod_{j = t + 1}^T \left( 1 + \eta_j \kappa \right) \right] 
        \frac{2\sqrt{2} \eta_t \alpha_\ell \sqrt{K} \|O\|}{m}.
    \end{aligned}
    \end{equation*}
    By the elementary inequality \( 1+a\leq\exp(a) \) and \( \eta_t\leq c/(t+1) \), we further derive
    \begin{equation*}
    \begin{aligned}
        \Delta_{t+1}
        &\leq \sum_{t = t_0 + 1}^T 
        \left[ \prod_{j = t + 1}^T \exp\left( \frac{c\kappa}{j} \right) \right] 
        \frac{2\sqrt{2} c\alpha_\ell \sqrt{K} \|O\|}{tm} \\
        &\leq \sum_{t = t_0 + 1}^T 
        \exp\left( \sum_{j = t + 1}^T \frac{c\kappa}{j} \right) 
        \frac{2\sqrt{2} c\alpha_\ell \sqrt{K} \|O\|}{tm} \\
        &\leq \sum_{t = t_0 + 1}^T 
        \exp\left( c\kappa \log\left( \frac{T}{t} \right) \right) 
        \frac{2\sqrt{2} c\alpha_\ell \sqrt{K} \|O\|}{tm} \\
        &\leq \frac{2\sqrt{2} c\alpha_\ell \sqrt{K} \|O\|}{m} T^{c\kappa} 
        \sum_{t = t_0 + 1}^T \frac{1}{t^{c\kappa + 1}} \\
        &\leq \frac{2\sqrt{2} \alpha_\ell \sqrt{K} \|O\|}{\kappa m} 
        \left( \frac{T}{t_0} \right)^{c\kappa}.
    \end{aligned}
    \end{equation*}
    Plugging this bound into equation (\ref{eq:4.3.1}), we get
    \begin{equation} \label{eq:4.3.3}
    \begin{aligned}
        \left| \mathbb{E}_A[\ell(f_{\boldsymbol{\theta}_T}(\boldsymbol{x}); y) - \ell(f_{\boldsymbol{\theta}_T^\prime}(\boldsymbol{x}); y)] \right|
        &\leq \frac{2\sqrt{2} \alpha_\ell^2 K \|O\|^2}{\kappa m} 
        \left( \frac{T}{t_0} \right)^{c\kappa} + \frac{t_0 M}{m}.
    \end{aligned}
    \end{equation}
    The right hand side is approximately minimized when
    \begin{equation*}
        t_0 = \left(\frac{2\sqrt{2}c\alpha_\ell^2K\|O\|^2}{M}\right)^{\frac{1}{c\kappa+1}} T^{\frac{c\kappa}{c\kappa+1}}.
    \end{equation*}
    Plugging it into equation (\ref{eq:4.3.3}) we have (for simplicity we assume the above \( t_0 \) is an integer)
     \begin{equation*} 
    \begin{aligned}
        \left| \mathbb{E}_A[\ell(f_{\boldsymbol{\theta}_T}(\boldsymbol{x}); y) - \ell(f_{\boldsymbol{\theta}_T^\prime}(\boldsymbol{x}); y)] \right|
        &\leq \frac{1 + \frac{1}{c\kappa}}{m}M^{\frac{c\kappa}{c\kappa + 1}} \left( 2\sqrt{2}c\alpha_\ell^2 K \|O\|^2 \right)^{\frac{1}{c\kappa + 1}} T^{\frac{c\kappa}{c\kappa + 1}}.
    \end{aligned}
    \end{equation*}
    By the definition of uniform stability as shown in Definition \ref{def:Uniform Stability}, we obtain the desired bound on the uniform stability under the decaying step size setting. This completes the proof of Corollary \ref{cor:1}, Part(b).
\end{proof}

\subsection{Proof of Corollary \ref{cor:2}}
In this subsection, we prove Corollary \ref{cor:2}, which extends previous analysis to noise scenario. We first introduce serval useful lemmas, as extensions of Lemma \ref{lem:C.1}, \ref{lem:C.2}, \ref{lem:C.3}, and \ref{lem:C.4}, under the depolarizing noise setting.
\renewcommand{\thetheorem}{C.5}
\begin{lemma}[From Loss Stability to Parameter Stability, Depolarizing Noise] \label{lem:C.5}
    Let \( \boldsymbol{\theta}_t \) and \( \boldsymbol{\theta}_t^\prime \) be the parameters of QNNs learned using the SGD algorithm for \( t \) iterations on training datasets \( S \) and \( S^\prime \), under depolarizing noise level \( p\in[0,1] \), respectively. Then, the output difference of the QNNs is bounded by, 
    \begin{equation*}
    \begin{aligned}
        \left| f_{\boldsymbol{\theta}_t}(\boldsymbol{x}) - f_{\boldsymbol{\theta}_t^\prime}(\boldsymbol{x}) \right| 
        &\leq (1-p)^{K_g}\sqrt{K} \| O \|\|\boldsymbol{\theta}_t-\boldsymbol{\theta}_t^\prime\|_2.
    \end{aligned}
    \end{equation*}
\end{lemma}
\begin{proof}
    We consider the noisy quantum channel is simulated by the local depolarization noise, i.e., the depolarization channel \( \mathcal{E}_p(\cdot) \) is applied to each quantum gate in \( U(\boldsymbol{\theta}) \). By Lemma \ref{lem:B.4}, the difference between the two output functions of the QNNs can be represented as follows,
    \begin{equation*}
    \begin{aligned}
        \left| f_{\boldsymbol{\theta}_t}(\boldsymbol{x}) - f_{\boldsymbol{\theta}_t^\prime}(\boldsymbol{x}) \right| 
        &= \left| \mathrm{Tr}\left( O \mathcal{E}_p\left(U(\boldsymbol{\theta}_t) \rho(\boldsymbol{x}) U^\dagger(\boldsymbol{\theta}_t) \right) \right) 
        - \mathrm{Tr}\left( O \mathcal{E}_p\left( U(\boldsymbol{\theta}_t^\prime) \rho(\boldsymbol{x}) U^\dagger (\boldsymbol{\theta}_t^\prime)\right)\right) \right| \\
        &=(1-p)^{K_g}\left| \mathrm{Tr}\left( O U(\boldsymbol{\theta}_t) \rho(\boldsymbol{x}) U^\dagger(\boldsymbol{\theta}_t) \right) 
        - \mathrm{Tr}\left( O U(\boldsymbol{\theta}_t^\prime) \rho(\boldsymbol{x}) U^\dagger (\boldsymbol{\theta}_t^\prime)\right) \right| \\
        &= (1-p)^{K_g}\left| \mathrm{Tr}\left( O \left( U(\boldsymbol{\theta}_t) \rho(\boldsymbol{x}) U^\dagger (\boldsymbol{\theta}_t)
        - U(\boldsymbol{\theta}_t^\prime) \rho(\boldsymbol{x}) U^\dagger(\boldsymbol{\theta}_t^\prime) \right) \right) \right| \\
        &\leq (1-p)^{K_g}\|O\| \left\| U(\boldsymbol{\theta}_t) \rho(\boldsymbol{x}) U^\dagger (\boldsymbol{\theta}_t)
        - U(\boldsymbol{\theta}_t^\prime) \rho(\boldsymbol{x}) U^\dagger(\boldsymbol{\theta}_t^\prime) \right\|_1,
    \end{aligned}
    \end{equation*}
    where the last inequality uses the Cauchy-Schwartz inequality.
    
    Next, by combining the proof technique used in Lemma \ref{lem:C.1}, we have
    \begin{equation*} 
    \begin{aligned}
         \left| f_{\boldsymbol{\theta}_t}(\boldsymbol{x}) - f_{\boldsymbol{\theta}_t^\prime}(\boldsymbol{x}) \right| 
         &\leq 2 (1-p)^{K_g} \|O\| \sum_{k=1}^K \left| 2 \sin\left( \frac{\theta_{t,k} - \theta_{t,k}^\prime}{4} \right) \right| \\
         &\leq 2 (1-p)^{K_g} \|O\| \sum_{k=1}^K \left| \frac{\theta_{t,k} - \theta_{t,k}^\prime}{2} \right| \\
         &\leq (1-p)^{K_g} \sqrt{K} \| O \|\|\boldsymbol{\theta}_t-\boldsymbol{\theta}_t^\prime\|_2.
    \end{aligned}
    \end{equation*}
    This completes the proof of Lemma \ref{lem:C.5}.
\end{proof}
\renewcommand{\thetheorem}{C.6}
\begin{lemma}[QNN Same Sample Loss Stability Bound, Depolarizing Noise]\label{lem:C.6}
     Suppose that Assumption \ref{ass:Lipschitzness} and \ref{ass:smoothness} hold. Let \( \boldsymbol{\theta}_t \) and \( \boldsymbol{\theta}_t^\prime \) be the parameters of two QNNs learned using the SGD algorithm for \( t \) iterations on two training datasets \( S \) and \( S^\prime \), under depolarizing noise level \( p\in[0,1] \), respectively. Then, the loss derivative difference of the QNNs with respect to the same sample is bounded by,
    \begin{equation*}
        \|\nabla_{\boldsymbol{\theta}} \ell(f_{\boldsymbol{\theta}_t}(\boldsymbol{x}); y) - \nabla_{\boldsymbol{\theta}} \ell(f_{\boldsymbol{\theta}_t^\prime}(\boldsymbol{x}); y)\|_2 \leq(1-p)^{K_g}\kappa\|\boldsymbol{\theta}_t-\boldsymbol{\theta}_t^\prime\|_2,
    \end{equation*}
    where \( \kappa=\alpha_\ell K \|O\| + \sqrt{2} \nu_\ell K \|O\|^2 \).
\end{lemma}
\begin{proof}
     According to equation (\ref{eq:C.2.1}) in Lemma \ref{lem:C.2}, we have
    \begin{equation} \label{eq:C.6.1}
    \begin{aligned}
        &\left\|\nabla_{\boldsymbol{\theta}} \ell\left(f_{\boldsymbol{\theta}_t}(\boldsymbol{x}); y\right) 
        - \nabla_{\boldsymbol{\theta}} \ell\left(f_{\boldsymbol{\theta}_t^\prime}(\boldsymbol{x}); y\right)\right\|_2 \\
        &\leq \alpha_\ell \left\|\nabla_{\boldsymbol{\theta}} f_{\boldsymbol{\theta}_t}(\boldsymbol{x}) 
        - \nabla_{\boldsymbol{\theta}} f_{\boldsymbol{\theta}_t^\prime}(\boldsymbol{x})\right\|_2 
        + \nu_\ell \left| f_{\boldsymbol{\theta}_t}(\boldsymbol{x}) - f_{\boldsymbol{\theta}_t^\prime}(\boldsymbol{x})\right| 
        \left\|\nabla_{\boldsymbol{\theta}} f_{\boldsymbol{\theta}_t^\prime}(\boldsymbol{x})\right\|_2.
    \end{aligned}
    \end{equation}
    Combing the equation (\ref{eq:C.2.2}) in Lemma \ref{lem:C.2} and Lemma \ref{lem:C.5}, we can bound the individual terms \( \left|\nabla_{\theta_j} f_{\boldsymbol{\theta}_t}(\boldsymbol{x}) - \nabla_{\theta_j} f_{\boldsymbol{\theta}_t^\prime}(\boldsymbol{x})\right| \) as follows
    \begin{equation}\label{eq:C.6.2}
    \begin{aligned}
        &\left|\nabla_{\theta_j} f_{\boldsymbol{\theta}_t}(\boldsymbol{x}) 
        - \nabla_{\theta_j} f_{\boldsymbol{\theta}_t^\prime}(\boldsymbol{x})\right| \\
        &\leq \frac{1}{2} \Big[ \left| f_{\boldsymbol{\theta}_t + \frac{\pi}{2} \boldsymbol{e}_j}(\boldsymbol{x}) 
        - f_{\boldsymbol{\theta}_t^\prime + \frac{\pi}{2} \boldsymbol{e}_j}(\boldsymbol{x}) \right| + \left| f_{\boldsymbol{\theta}_t - \frac{\pi}{2} \boldsymbol{e}_j}(\boldsymbol{x}) 
        - f_{\boldsymbol{\theta}_t^\prime - \frac{\pi}{2} \boldsymbol{e}_j}(\boldsymbol{x}) \right| \Big] \\
        &\leq (1-p)^{K_g} \sqrt{K} \| O \|\|\boldsymbol{\theta}_t-\boldsymbol{\theta}_t^\prime\|_2.
    \end{aligned}
    \end{equation}
    Therefore, the term \( \left\|\nabla_{\boldsymbol{\theta}} f_{\boldsymbol{\theta}_t}(\boldsymbol{x}) - \nabla_{\boldsymbol{\theta}} f_{\boldsymbol{\theta}_t^\prime}(\boldsymbol{x})\right\|_2 \) can be bounded as
    \begin{equation}\label{eq:C.6.3}
    \begin{aligned}
        \|\nabla_{\boldsymbol{\theta}} f_{\boldsymbol{\theta}_t}(\boldsymbol{x}) 
        - \nabla_{\boldsymbol{\theta}} f_{\boldsymbol{\theta}_t^\prime}(\boldsymbol{x})\|_2
        &= \sqrt{\sum_{j=1}^K \left| \nabla_{\theta_j} f_{\boldsymbol{\theta}_t}(\boldsymbol{x}) 
        - \nabla_{\theta_j} f_{\boldsymbol{\theta}_t^\prime}(\boldsymbol{x}) \right|^2} \\
        &\leq \sqrt{K \left( (1-p)^{K_g} \sqrt{K} \| O \|\|\boldsymbol{\theta}_t-\boldsymbol{\theta}_t^\prime\|_2 \right)^2} \\
        &= (1-p)^{K_g} K \|O\| \|\boldsymbol{\theta}_t - \boldsymbol{\theta}_t^\prime\|_2.
    \end{aligned}
    \end{equation}
    Combing the equation (\ref{eq:C.2.4}) in Lemma \ref{lem:C.2} and Lemma \ref{lem:C.5}, we can similarly bound the individual terms \( \left| \nabla_{\theta_j} f_{\boldsymbol{\theta}_t^\prime}(\boldsymbol{x}) \right| \) as follows
    \begin{equation}\label{eq:C.6.4}
    \begin{aligned}
        \left| \nabla_{\theta_j} f_{\boldsymbol{\theta}_t^\prime}(\boldsymbol{x}) \right| 
        &\leq (1-p)^{K_g} \|O\| \left( \sum_{k \neq j} \left| 2 \sin\left( \frac{\theta_{t,k}^\prime - \theta_{t,k}^\prime}{4} \right) \right| 
        + \left| 2 \sin\left( \frac{(\theta_{t,j}^\prime + \frac{\pi}{2}) - (\theta_{t,j}^\prime - \frac{\pi}{2})}{4} \right) \right| \right) \\
        &= \sqrt{2} (1-p)^{K_g} \|O\|.
    \end{aligned}
    \end{equation}
    Therefore, the term \(  \left\|\nabla_{\boldsymbol{\theta}} f_{\boldsymbol{\theta}_t^\prime}(\boldsymbol{x})\right\|_2 \) can be bounded as
    \begin{equation}\label{eq:C.6.5}
    \begin{aligned}
        \|\nabla_{\boldsymbol{\theta}} f_{\boldsymbol{\theta}_t^\prime}(\boldsymbol{x})\|_2
        &= \sqrt{\sum_{j=1}^K \left|\nabla_{\theta_j} f_{\boldsymbol{\theta}_t^\prime}(\boldsymbol{x})\right|^2} \\
        &\leq \sqrt{K \left(\sqrt{2} (1-p)^{K_g} \|O\|\right)^2} \\
        &= \sqrt{2K} (1-p)^{K_g} \|O\|.
    \end{aligned}
    \end{equation}
    Finally, by Lemma \ref{lem:C.5}, the term \(  \left| f_{\boldsymbol{\theta}_t}(\boldsymbol{x}) - f_{\boldsymbol{\theta}_t^\prime}(\boldsymbol{x})\right| \) can be bounded as
    \begin{equation} \label{eq:C.6.6}
        \left| f_{\boldsymbol{\theta}_t}(\boldsymbol{x}) - f_{\boldsymbol{\theta}_t^\prime}(\boldsymbol{x})\right| \leq (1-p)^{K_g} \sqrt{K}  \|O\| \|\boldsymbol{\theta}_t - \boldsymbol{\theta}_t^\prime\|_2
    \end{equation}
    Plugging equation (\ref{eq:C.6.3}), (\ref{eq:C.6.5}), and (\ref{eq:C.6.6}) back into equation (\ref{eq:C.2.1}), we further get
    \begin{equation*} 
    \begin{aligned}
        \|\nabla_{\boldsymbol{\theta}} \ell(f_{\boldsymbol{\theta}_t}(\boldsymbol{x}); y) - \nabla_{\boldsymbol{\theta}} \ell(f_{\boldsymbol{\theta}_t^\prime}(\boldsymbol{x}); y)\|_2 
        &\leq (1-p)^{K_g} \alpha_\ell K \|O\| + (1-p)^{2K_g} \sqrt{2} \nu_\ell K \|O\|^2 \\
        &\leq (1-p)^{K_g}\kappa, 
    \end{aligned}
    \end{equation*}
    where \( \kappa=\alpha_\ell K \|O\| + \sqrt{2} \nu_\ell K \|O\|^2 \).
    
    This completes the proof of Lemma \ref{lem:C.6}.
\end{proof}
\renewcommand{\thetheorem}{C.7}
\begin{lemma}[QNN Differnet Sample Loss Stability Bound, Depolarizing Noise] \label{lem:C.7}
     Suppose that Assumption \ref{ass:Lipschitzness} hold. Let \( \boldsymbol{\theta}_t \) and \( \boldsymbol{\theta}_t^\prime \) be the parameters of two QNNs learned using the SGD algorithm for \( t \) iterations on two training datasets \( S \) and \( S^\prime \), under depolarizing noise level \( p\in[0,1] \), respectively. Then, the loss derivative difference of the QNNs with respect to the different sample is bounded by,
    \begin{equation*}
        \|\nabla_{\boldsymbol{\theta}} \ell(f_{\boldsymbol{\theta}_t}(\boldsymbol{x}); y) - \nabla_{\boldsymbol{\theta}} \ell(f_{\boldsymbol{\theta}_t^\prime}(\boldsymbol{x}^\prime); y^\prime)\|_2 \leq2\sqrt{2}(1-p)^{K_g}\alpha_\ell\sqrt{K}\|O\|.
    \end{equation*}
\end{lemma}
\begin{proof}
    According to equation (\ref{eq:C.3.1}) in Lemma \ref{lem:C.3}, we have
    \begin{equation}\label{eq:C.7.1}
    \begin{aligned}
        \left\|\nabla_{\boldsymbol{\theta}} \ell\left(f_{\boldsymbol{\theta}_t}(\boldsymbol{x}); y\right) 
        - \nabla_{\boldsymbol{\theta}} \ell\left(f_{\boldsymbol{\theta}_t^\prime}(\boldsymbol{x}^\prime); y^\prime\right)\right\|_2 
        &\leq \alpha_\ell \left(
        \left\|\nabla_{\boldsymbol{\theta}} f_{\boldsymbol{\theta}_t}(\boldsymbol{x})\right\|_2 
        + \left\|\nabla_{\boldsymbol{\theta}} f_{\boldsymbol{\theta}_t^\prime}(\boldsymbol{x}^\prime)\right\|_2
        \right).
    \end{aligned}
    \end{equation}
    By equation (\ref{eq:C.6.5}) in Lemma \ref{lem:C.6}, we similarly bound \( \left\|\nabla_{\boldsymbol{\theta}} f_{\boldsymbol{\theta}_t}(\boldsymbol{x})\right\|_2 \leq \sqrt{2K} (1-p)^{K_g} \|O\|, \left\|\nabla_{\boldsymbol{\theta}} f_{\boldsymbol{\theta}_t^\prime}(\boldsymbol{x}^\prime)\right\|_2 \leq \sqrt{2K} (1-p)^{K_g} \|O\| \). Plugging this into equation (\ref{eq:C.7.1}) completes the proof of Lemma \ref{lem:C.3}.  
\end{proof}
\renewcommand{\thetheorem}{C.8}
\begin{lemma} \label{lem:C.8}
    Suppose that Assumption \ref{ass:Lipschitzness} and \ref{ass:smoothness} hold, and \( \ell(\cdot,\cdot)\in[0,M]\). Let \( S \) and \( S^\prime \) of size \( m \) differing in only a single example. Consider two sequences of parameters, \( \{\boldsymbol{\theta}_0,\boldsymbol{\theta}_1,\ldots,\boldsymbol{\theta}_T\}\) and \( \{\boldsymbol{\theta}_0^\prime,\boldsymbol{\theta}_1^\prime,\ldots,\boldsymbol{\theta}_T^\prime \} \), learned by the QNN running SGD on \( S \) and \( S^\prime \), under depolarizing noise level \( p\in[0,1] \), respectively. Let \( \delta_t=\|\boldsymbol{\theta}_t-\boldsymbol{\theta}_t^\prime\|_2 \). Then, for any \( \boldsymbol{z}\in\mathcal{Z} \) and \( t_0\in\{0,1,\ldots,m\} \), we have
    \begin{equation*}
    \begin{aligned}
        \left| \mathbb{E}_A[\ell(f_{\boldsymbol{\theta}_T}(\boldsymbol{x}); y) - \ell(f_{\boldsymbol{\theta}_T^\prime}(\boldsymbol{x}); y)] \right|
        &\leq (1-p)^{K_g} \alpha_\ell \sqrt{K} \|O\| \mathbb{E}_A[\delta_T \mid \delta_{t_0}=0] + \frac{t_0 M}{m}.
    \end{aligned}
    \end{equation*}
\end{lemma}
\begin{proof}
     Let \( \mathcal{E} \) denote the event that \( \delta_{t_0}=0 \). By equation (\ref{eq:C.4.1}) and  (\ref{eq:C.4.3}) in Lemma \ref{lem:C.4}, we have
    \begin{equation}\label{eq:C.8.1}
    \begin{aligned}
        |\mathbb{E}[\ell(f_{\boldsymbol{\theta}_T}(\boldsymbol{x}); y) - \ell(f_{\boldsymbol{\theta}_T^\prime}(\boldsymbol{x}); y)]|
        &\leq \mathbb{E}\left[ \left| \ell(f_{\boldsymbol{\theta}_T}(\boldsymbol{x}); y) - \ell(f_{\boldsymbol{\theta}_T^\prime}(\boldsymbol{x}); y) \right| \mid \mathcal{E} \right] + \frac{t_0 M}{m}.
    \end{aligned}
    \end{equation}
    Using the Assumption \ref{ass:Lipschitzness} that the loss function is Lipschitz continuous and Lemma \ref{lem:C.5}, the first term \(  \mathbb{E}\left[ \left| \ell(f_{\boldsymbol{\theta}_T}(\boldsymbol{x}); y) - \ell(f_{\boldsymbol{\theta}_T^\prime}(\boldsymbol{x}); y) \right| \mid \mathcal{E} \right] \) can be bounded as
    \begin{equation}\label{eq:C.7.2}
        \mathbb{E}\left[ \left| \ell(f_{\boldsymbol{\theta}_T}(\boldsymbol{x}); y) - \ell(f_{\boldsymbol{\theta}_T^\prime}(\boldsymbol{x}); y) \right| \mid \mathcal{E} \right] \leq (1-p)^{K_g} \alpha_\ell \sqrt{K} \|O\| \mathbb{E}[\delta_T \mid \mathcal{E}].
    \end{equation}
    Plugging it into equation (\ref{eq:C.8.1}) completes the proof of Lemma \ref{lem:C.8}.
\end{proof}
We now are ready to prove Corollary \ref{cor:2}.
\renewcommand{\thetheorem}{\ref{cor:2}}
\begin{corollary}[Generalization Bound Under Depolarizing Noise]
    Suppose that Assumption \ref{ass:Lipschitzness} and \ref{ass:smoothness} hold, and \( \ell(\cdot,\cdot)\in[0,M] \). Let \( A(S) \) be the QNN model trained on the dataset \( S\in\mathcal{Z}^m \) using the SGD algorithm with step sizes \( \eta_t\) under depolarizing noise level \( p\in[0,1] \) for \( T \) iterations. 
    \begin{enumerate}
        \item [(a)] if we choose the constant step sizes \( \eta_t=\eta \), then the following generalization bound of \(
        A(S) \) holds with probability at least \( 1-\delta \) for \( \delta\in(0,1) \),
        \begin{equation*}
            \begin{aligned}
                &\mathbb{E}_A\left[R_{\mathcal{D}}\left(A(S)\right)-R_S\left(A(S)\right)\right] \\
                &\leq \mathcal{O}\left(\frac{\left(1+\left(1-p\right)^{K_g}\eta\kappa\right)^T}{m}\log m \log(\frac{1}{\delta})+M\sqrt{\frac{\log(\frac{1}{\delta})}{m}}\right),
            \end{aligned}
        \end{equation*}
        \item [(b)] if we choose the monotonically non-increasing step sizes \( \eta_t\leq c/(t+1) \), \( c>0 \), then the following generalization bound of \(
        A(S) \) holds with probability at least \( 1-\delta \) for \( \delta\in(0,1) \),
          \begin{equation*}
            \begin{aligned}
                &\mathbb{E}_A\left[R_{\mathcal{D}}\left(A(S)\right)-R_S\left(A(S)\right)\right] \\
                &\leq \mathcal{O}\left(\frac{T^{\frac{c\kappa}{c\kappa+1/\left(1-p\right)^{K_g}}}}{m}\log m \log(\frac{1}{\delta})+M\sqrt{\frac{\log(\frac{1}{\delta})}{m}}\right),
            \end{aligned}
        \end{equation*} 
    \end{enumerate}
    where \( \kappa \) is defined by equation (\ref{eq:k}).
\end{corollary}
\begin{proof}
    Let \( S \) and \( S^\prime\) be two datasets of size \( m \) differing in only a single sample. Consider two sequences of the parameters, \( \{\boldsymbol{\theta}_0,\boldsymbol{\theta}_1,\ldots,\boldsymbol{\theta}_T\}\) and \( \{\boldsymbol{\theta}_0^\prime,\boldsymbol{\theta}_1^\prime,\ldots,\boldsymbol{\theta}_T^\prime \} \), learned by the QNN running SGD on \( S \) and \( S^\prime \), respectively. Let \( \delta_t=\|\boldsymbol{\theta}_t-\boldsymbol{\theta}_t^\prime\|_2 \). 

    For the constant step sizes \( \eta_t=\eta \), using the Assumption \ref{ass:Lipschitzness} that the loss function is Lipschitz continuous, the linearity of expectation and Lemma \ref{lem:C.5}, we have
    \begin{equation}\label{eq:4.8.1}
    \begin{aligned}
        \left| \mathbb{E}_A \left[ \ell (f_{\boldsymbol{\theta}_T}(\boldsymbol{x}); y) - \ell (f_{\boldsymbol{\theta}_T^\prime}(\boldsymbol{x}); y) \right] \right|
        &\leq \alpha_\ell \mathbb{E}_A \left[ \left| f_{\boldsymbol{\theta}_T}(\boldsymbol{x}) - f_{\boldsymbol{\theta}_T^\prime}(\boldsymbol{x}) \right| \right] \\
        &\leq (1-p)^{K_g} \alpha_\ell \sqrt{K} \|O\| \mathbb{E}_A\left[ \|\boldsymbol{\theta}_T - \boldsymbol{\theta}_T^\prime\|_2 \right] \\
        &\leq (1-p)^{K_g} \alpha_\ell \sqrt{K} \|O\| \mathbb{E}_A\left[ \delta_T \right].
    \end{aligned}
    \end{equation} 
    Combining equation (\ref{eq:4.1.2}), Lemma \ref{lem:C.6}, and Lemma \ref{lem:C.7}, we further get
    \begin{equation*}
    \begin{aligned}
        \mathbb{E}_A[\delta_{t+1}] 
        &\leq \mathbb{E}_A[\delta_{t}] + (1 - \frac{1}{m}) \eta \cdot \mathbb{E}_A[(1-p)^{K_g}\kappa \delta_t] + \frac{1}{m} \eta \cdot \mathbb{E}_A[2\sqrt{2} (1-p)^{K_g} \alpha_\ell \sqrt{K} \|O\|] \\
        &\leq (1 + (1-p)^{K_g} \eta \kappa) \mathbb{E}_A[\delta_t] + \frac{2\sqrt{2} (1-p)^{K_g} \eta \alpha_\ell \sqrt{K} \|O\|}{m}.
    \end{aligned}
    \end{equation*}
    Unraveling the recursion gives
    \begin{equation*}
    \begin{aligned}
        \mathbb{E}_A[\delta_T] 
        &\leq \sum_{t=0}^{T-1} \left[ \prod_{j=t+1}^{T-1} \left( 1 + (1-p)^{K_g} \eta \kappa \right) \right] \frac{2\sqrt{2} (1-p)^{K_g} \eta \alpha_\ell \sqrt{K} \|O\|}{m} \\
        &\leq \frac{2\sqrt{2}(1-p)^{K_g}\eta\alpha_\ell\sqrt{K}\|O\|}{m}\sum_{t=0}^{T-1}(1+(1-p)^{K_g}\eta\kappa)^t \\
        &\leq \frac{2\sqrt{2}(1-p)^{K_g}\eta\alpha_\ell\sqrt{K}\|O\|}{m} \cdot \frac{(1+(1-p)^{K_g}\eta\kappa)^T - 1}{(1-p)^{K_g}\eta\kappa}  \\
        &\leq \frac{2\sqrt{2}\alpha_\ell\sqrt{K}\|O\|}{\kappa m}(1+(1-p)^{K_g}\eta\kappa)^T.
    \end{aligned}
    \end{equation*}
    Plugging it into equation (\ref{eq:4.8.1}), we obtain
    \begin{equation*}
         \left| \mathbb{E}_A \left[ \ell (f_{\boldsymbol{\theta}_T}(\boldsymbol{x}); y) - \ell (f_{\boldsymbol{\theta}_T^\prime}(\boldsymbol{x}); y) \right] \right| \leq \frac{2\sqrt{2}\alpha_\ell^2K\|O\|^2}{\kappa m}(1+(1-p)^{K_g}\eta\kappa)^T.
    \end{equation*}
    By the definition of uniform stability as shown in Definition \ref{def:Uniform Stability}, we obtain the desired bound on the uniform stability of QNNs under the constant step size and depolarizing noise setting. Combining Lemma \ref{lem:Stability and Generalization}, Part (b) completes the proof of Corollary \ref{cor:2}, Part (a).
    
    For the decaying step sizes \( \eta_t\leq c/(t+1) \), by Lemma \ref{lem:C.8}, we have for every \( t_0\in\{0,1,\ldots,m\} \),
    \begin{equation} \label{eq:4.8.2}
    \begin{aligned}
        \left| \mathbb{E}_A[\ell(f_{\boldsymbol{\theta}_T}(\boldsymbol{x}); y) - \ell(f_{\boldsymbol{\theta}_T^\prime}(\boldsymbol{x}); y)] \right|
        &\leq (1-p)^{K_g} \alpha_\ell \sqrt{K} \|O\| \mathbb{E}_A[\delta_T \mid \delta_{t_0}=0] + \frac{t_0 M}{m}.
    \end{aligned}
    \end{equation}
    Let \( \Delta_t = \mathbb{E}[\delta_t \mid \delta_{t_0} = 0] \). Combining equation (\ref{eq:4.3.2}), Lemma \ref{lem:C.6}, and Lemma \ref{lem:C.7}, we further get
   \begin{equation*}
    \begin{aligned}
        \Delta_{t+1}  
        &\leq \Delta_{t} 
        + \left(1 - \frac{1}{m}\right) \eta_t \cdot \mathbb{E}_A \big[(1-p)^{K_g}\kappa \delta_t \;\big|\; \delta_{t_0} = 0 \big] 
        + \frac{1}{m} \eta_t \cdot \mathbb{E}_A \big[2\sqrt{2}(1-p)^{K_g} \alpha_\ell \sqrt{K} \|O\| \;\big|\; \delta_{t_0} = 0 \big] \\
        &\leq (1 + (1-p)^{K_g}\eta_t \kappa) \Delta_{t} 
        + \frac{2\sqrt{2} (1-p)^{K_g}\eta_t \alpha_\ell \sqrt{K} \|O\|}{m}.
    \end{aligned}
    \end{equation*}
    Using the fact that \( \Delta_{t_0} = 0 \), we can unwind this recurrence relation from \( T \) down to \(t_0 + 1\). This gives
   \begin{equation*}
    \begin{aligned}
        \Delta_{t+1}
        &= \sum_{t = t_0 + 1}^T 
        \left[ \prod_{j = t + 1}^T \left( 1 + (1-p)^{K_g}\eta_j \kappa \right) \right] 
        \frac{2\sqrt{2} (1-p)^{K_g}\eta_t \alpha_\ell \sqrt{K} \|O\|}{m}.
    \end{aligned}
    \end{equation*}
    By the elementary inequality \( 1+a\leq\exp(a) \) and \( \eta_t\leq c/(t+1) \), we further derive
    \begin{equation*}
    \begin{aligned}
        \Delta_{t+1}
        &\leq \sum_{t = t_0 + 1}^T 
        \left[ \prod_{j = t + 1}^T \exp\left( \frac{c (1-p)^{K_g}\kappa}{j} \right) \right] 
        \frac{2\sqrt{2} c (1-p)^{K_g} \alpha_\ell \sqrt{K} \|O\|}{tm} \\
        &\leq \sum_{t = t_0 + 1}^T 
        \exp\left( \sum_{j = t + 1}^T \frac{c (1-p)^{K_g} \kappa}{j} \right) 
        \frac{2\sqrt{2} c (1-p)^{K_g}\alpha_\ell \sqrt{K} \|O\|}{tm} \\
        &\leq \sum_{t = t_0 + 1}^T 
        \exp\left( c (1-p)^{K_g} \kappa \log\left( \frac{T}{t} \right) \right) 
        \frac{2\sqrt{2} c (1-p)^{K_g} \alpha_\ell \sqrt{K} \|O\|}{tm} \\
        &\leq \frac{2\sqrt{2} c (1-p)^{K_g} \alpha_\ell \sqrt{K} \|O\|}{m} T^{c (1-p)^{K_g} \kappa} 
        \sum_{t = t_0 + 1}^T \frac{1}{t^{c (1-p)^{K_g} \kappa + 1}} \\
        &\leq \frac{2\sqrt{2} \alpha_\ell \sqrt{K} \|O\|}{\kappa m} 
        \left( \frac{T}{t_0} \right)^{c (1-p)^{K_g} \kappa}.
    \end{aligned}
    \end{equation*}
    Plugging this bound into equation (\ref{eq:4.8.2}), we get
    \begin{equation} 
    \begin{aligned}
        \left| \mathbb{E}_A[\ell(f_{\boldsymbol{\theta}_T}(\boldsymbol{x}); y) - \ell(f_{\boldsymbol{\theta}_T^\prime}(\boldsymbol{x}); y)] \right|
        &\leq \frac{2\sqrt{2} (1-p)^{K_g} \alpha_\ell^2 K \|O\|^2}{\kappa m} 
        \left( \frac{T}{t_0} \right)^{c (1-p)^{K_g} \kappa} + \frac{t_0 M}{m}.
    \end{aligned}
    \end{equation}
    The right hand side is approximately minimized when
    \begin{equation*}
        t_0 = \left( \frac{2\sqrt{2}c(1-p)^{K_g}\alpha_\ell^2 K \|O\|^2}{M} \right)^{\frac{1}{c(1-p)^{K_g}\kappa + 1}} 
        T^{\frac{c\kappa}{c\kappa + 1/(1-p)^{K_g}}}.
    \end{equation*}
    Plugging it into the equation (\ref{eq:4.8.2}) we have (for simplicity we assume the above \( t_0 \) is an integer)
    \begin{equation*} 
    \begin{aligned}
        &\left| \mathbb{E}_A[\ell(f_{\boldsymbol{\theta}_T}(\boldsymbol{x}); y) - \ell(f_{\boldsymbol{\theta}_T^\prime}(\boldsymbol{x}); y)] \right| \\
        &\leq \frac{1 + \frac{1}{c(1-p)^{K_g}\kappa}}{m}
        M^{\frac{c\kappa}{c\kappa + 1/(1-p)^{K_g}}} 
        \left( 2\sqrt{2}c(1-p)^{K_g}\alpha_\ell^2 K \|O\|^2 \right)^{\frac{1}{c(1-p)^{K_g}\kappa + 1}} 
        T^{\frac{c\kappa}{c\kappa + 1/(1-p)^{K_g}}}.
    \end{aligned}
    \end{equation*}
    By the definition of uniform stability as shown in Definition \ref{def:Uniform Stability}, we obtain the desired bound on the uniform stability of QNNs under the decaying step size and depolarizing noise setting. Combining Lemma \ref{lem:Stability and Generalization}, Part (b) completes the proof of Corollary \ref{cor:2}, Part (b).
\end{proof}

\subsection{Proof of Theorem \ref{the:3}}
In this subsection, we prove Theorem \ref{the:3}, which establishes the generalization bound in expectation for QNNs and first links generalization gap to optimization error with the help of less restrictive on-average stability. We begin by quoting the result which we set to prove.
\renewcommand{\thetheorem}{\ref{the:3}}
\begin{theorem}[Optimization-Dependent Generalization Bound]
    Suppose that Assumption \ref{ass:Lipschitzness}, \ref{ass:smoothness}, and \ref{ass:Variance} hold. Let \( A(S) \) be the QNN model trained on the dataset \( S\in\mathcal{Z}^m \) using the SGD algorithm with step sizes \( \eta_t\) for \( T \) iterations, then the expected generalization gap satisfies
    \begin{equation*}
        \begin{aligned}
            & \left|\mathbb{E}_{S,A} \left[ R_{\mathcal{D}}\left(A(S)\right) - R_S\left(A(S)\right) \right]\right| \\
            &\leq \sum_{t=0}^{T-1}\left[\prod_{j=t+1}^{T-1}(1+\eta_t\kappa)\right]\frac{2\eta_t\alpha_\ell\sqrt{K}\|O\|\left(\mathbb{E}_S[\|\nabla_{\boldsymbol{\theta}} R_S(\boldsymbol{\theta}_t)\|_2]+\sigma\right)}{m},
        \end{aligned}
    \end{equation*}
    where \( \kappa \) is defined by equation (\ref{eq:k}).
\end{theorem}
\begin{proof}
    Let \( S=\{\boldsymbol{z}_1,\ldots,\boldsymbol{z}_m\} \) and \( S^\prime=\{\boldsymbol{z}_1^\prime,\ldots,\boldsymbol{z}_m^\prime\} \) be drawn independently from \( \mathcal{D} \). For any \(i\in[m]\), define \(S^{(i)}=\{\boldsymbol{z}_1,\ldots,\boldsymbol{z}_{i-1},\boldsymbol{z}_i^\prime,\boldsymbol{z}_{i+1},\ldots,\boldsymbol{z}_m\}\) as the set formed from \( S \) by replacing the \( i \)-th element with \( \boldsymbol{z}_i^\prime\). Consider two sequences of the parameters, \( \{\boldsymbol{\theta}_0,\boldsymbol{\theta}_1,\ldots,\boldsymbol{\theta}_T\}\) and \( \{\boldsymbol{\theta}_0^{(i)},\boldsymbol{\theta}_1^{(i)},\ldots,\boldsymbol{\theta}_T^{(i)} \} \), learned by the QNN running SGD on \( S \) and \( S^{(i)} \), respectively. Let the example index selected by SGD at iteration \( t \) denoted by \( i_t \).
    
    Observe that at iteration \( t \), with probability \( 1-1/m \), \( i_t\neq i\), the example selected by SGD is the same in both \( S \) and \( S^{(i)}\). With probability \( 1/m \), \( i_t= i\), the selected example is different. Therefore, we have
    \begin{equation*}
    \begin{aligned}
    \mathbb{E}_A[\|\boldsymbol{\theta}_{t+1}-\boldsymbol{\theta}_{t+1}^{(i)}\|_2] 
    &\leq (1 - \frac{1}{m}) \mathbb{E}_A \left[ \left\| ( \boldsymbol{\theta}_{t} - \eta_t \nabla_{\boldsymbol{\theta}} \ell(f_{\boldsymbol{\theta}_t}(\boldsymbol{x}_{i_t}); y_{i_t}) ) - ( \boldsymbol{\theta}_{t}^{(i)} - \eta_t \nabla_{\boldsymbol{\theta}} \ell(f_{\boldsymbol{\theta}_t^{(i)}}(\boldsymbol{x}_{i_t}); y_{i_t}) ) \right\|_2 \right] \\
    & \quad + \frac{1}{m} \mathbb{E}_A \left[ \left\| ( \boldsymbol{\theta}_{t} - \eta_t \nabla_{\boldsymbol{\theta}} \ell(f_{\boldsymbol{\theta}_t}(\boldsymbol{x}_i); y_i) ) - ( \boldsymbol{\theta}_{t}^{(i)} - \eta_t \nabla_{\boldsymbol{\theta}} \ell(f_{\boldsymbol{\theta}_t^{(i)}}(\boldsymbol{x}_i^\prime); y_i^\prime ) \right\|_2 \right] \\
    &= \mathbb{E}_A[\|\boldsymbol{\theta}_t-\boldsymbol{\theta}_t^{(i)}\|_2] + (1 - \frac{1}{m}) \eta_t \mathbb{E}_A \left[ \left\| \nabla_{\boldsymbol{\theta}} \ell(f_{\boldsymbol{\theta}_t}(\boldsymbol{x}_{i_t}); y_{i_t}) - \nabla_{\boldsymbol{\theta}} \ell(f_{\boldsymbol{\theta}_t^{(i)}}(\boldsymbol{x}_{i_t}); y_{i_t}) \right\|_2 \right] \\
    & \quad + \frac{1}{m} \eta_t \mathbb{E}_A \left[ \left\| \nabla_{\boldsymbol{\theta}} \ell(f_{\boldsymbol{\theta}_t}(\boldsymbol{x}_i); y_i) - \nabla_{\boldsymbol{\theta}} \ell(f_{\boldsymbol{\theta}_t^{(i)}}(\boldsymbol{x}_i^\prime); y_i^\prime) \right\|_2 \right].
    \end{aligned}
    \end{equation*}
    According to Lemma \ref{lem:C.2} and Lemma \ref{lem:C.3}, we further get
    \begin{equation*}
    \begin{aligned}
        \mathbb{E}_A[\|\boldsymbol{\theta}_{t+1}-\boldsymbol{\theta}_{t+1}^{(i)}\|_2] 
        &\leq \mathbb{E}_A[\|\boldsymbol{\theta}_t-\boldsymbol{\theta}_t^{(i)}\|_2] + (1 - \frac{1}{m}) \eta_t \cdot \mathbb{E}_A[\kappa \|\boldsymbol{\theta}_t-\boldsymbol{\theta}_t^{(i)}\|_2] \\
        & \quad + \frac{1}{m} \eta_t \cdot \mathbb{E}_A[\left\| \nabla_{\boldsymbol{\theta}} \ell(f_{\boldsymbol{\theta}_t}(\boldsymbol{x}_i); y_i)\right\|_2 + \|\nabla_{\boldsymbol{\theta}} \ell(f_{\boldsymbol{\theta}_t^{(i)}}(\boldsymbol{x}_i^\prime); y_i^\prime) \|_2] \\
        &\leq (1 + \eta_t \kappa) \mathbb{E}_A[\|\boldsymbol{\theta}_t-\boldsymbol{\theta}_t^{(i)}\|_2] + \frac{\eta_t(\left\| \nabla_{\boldsymbol{\theta}} \ell(f_{\boldsymbol{\theta}_t}(\boldsymbol{x}_i); y_i)\right\|_2 + \|\nabla_{\boldsymbol{\theta}} \ell(f_{\boldsymbol{\theta}_t^{(i)}}(\boldsymbol{x}_i^\prime); y_i^\prime) \|_2)}{m}.
    \end{aligned}
    \end{equation*}
    Unraveling the recursion gives
    \begin{equation*}
    \begin{aligned}
        \mathbb{E}_A[\|\boldsymbol{\theta}_T-\boldsymbol{\theta}_T^{(i)}\|_2] 
        &\leq \sum_{t=0}^{T-1} \left[ \prod_{j=t+1}^{T-1} \left( 1 + \eta_j \kappa \right) \right] 
        \frac{\eta_t(\left\| \nabla_{\boldsymbol{\theta}} \ell(f_{\boldsymbol{\theta}_t}(\boldsymbol{x}_i); y_i)\right\|_2 + \|\nabla_{\boldsymbol{\theta}} \ell(f_{\boldsymbol{\theta}_t^\prime}(\boldsymbol{x}_i^\prime); y_i^\prime) \|_2)}{m} \\
        &\leq \sum_{t=0}^{T-1} \left[ \prod_{j=t+1}^{T-1} \left( 1 + \eta_j \kappa \right) \right] 
        \frac{2\eta_t\left\| \nabla_{\boldsymbol{\theta}} \ell(f_{\boldsymbol{\theta}_t}(\boldsymbol{x}_i); y_i)\right\|_2}{m}.
    \end{aligned}
    \end{equation*}
    Using the Assumption \ref{ass:Variance} and take an average over \(i\in[m] \), we have
    \begin{equation*}
    \begin{aligned}
        \frac{1}{m}\sum_{i=1}^m\mathbb{E}_A[\|\boldsymbol{\theta}_T-\boldsymbol{\theta}_T^{(i)}\|_2] 
        &\leq \sum_{t=0}^{T-1} \left[ \prod_{j=t+1}^{T-1} \left( 1 + \eta_j \kappa \right) \right] 
        \frac{2\eta_t \sum_{i=1}^m\left\| \nabla_{\boldsymbol{\theta}} \ell(f_{\boldsymbol{\theta}_t}(\boldsymbol{x}_i); y_i)\right\|_2}{m^2} \\
        &\leq \sum_{t=0}^{T-1} \left[ \prod_{j=t+1}^{T-1} \left( 1 + \eta_j \kappa \right) \right] 
        \frac{2\eta_t \mathbb{E}_S[\left\| \nabla_{\boldsymbol{\theta}} \ell(f_{\boldsymbol{\theta}_t}(\boldsymbol{x}_i); y_i)\right\|_2]}{m} \\
        &\leq \sum_{t=0}^{T-1} \left[ \prod_{j=t+1}^{T-1} \left( 1 + \eta_j \kappa \right) \right] 
        \frac{2\eta_t \left( \mathbb{E}_S\left[\|\nabla_{\boldsymbol{\theta}} R_S(\boldsymbol{\theta}_t)\|_2\right] 
        + \mathbb{E}_S\left[\|\nabla_{\boldsymbol{\theta}} \ell(f_{\boldsymbol{\theta}_t}(\boldsymbol{x}_i); y_i) 
        - \nabla_{\boldsymbol{\theta}} R_S(\boldsymbol{\theta}_t)\|_2\right] \right)}{m} \\
        &\leq \sum_{t=0}^{T-1} \left[ \prod_{j=t+1}^{T-1} \left( 1 + \eta_j \kappa \right) \right] 
        \frac{2\eta_t \left( \mathbb{E}_S\left[\|\nabla_{\boldsymbol{\theta}} R_S(\boldsymbol{\theta}_t)\|_2\right] 
        + \sigma \right)}{m}.
    \end{aligned}
    \end{equation*}
    By the definition of on-average stability as shown in Definition \ref{def:On-Average Stability}, we immediately get the claimed upper bound on the on-average stability of QNNs.
    
    Using the Assumption \ref{ass:Lipschitzness} that the loss function is Lipschitz continuous and Lemma \ref{lem:C.1}, we have for \(\forall{\boldsymbol{\theta}_1,\boldsymbol{\theta}_2}\in\mathbb{R}^K\)
    \begin{equation}\label{eq:4.12.1}
    \begin{aligned}
        \left|\ell \left( f_{\boldsymbol{\theta}_1}(\boldsymbol{x}); y \right) - \ell \left( f_{\boldsymbol{\theta}_2}(\boldsymbol{x}); y \right) \right|
        &\leq \alpha_\ell \left| f_{\boldsymbol{\theta}_1}(\boldsymbol{x}) - f_{\boldsymbol{\theta}_2}(\boldsymbol{x}) \right| \\
        &\leq \alpha_\ell \sqrt{K} \|O\| \|\boldsymbol{\theta}_1 - \boldsymbol{\theta}_2\|_2.
    \end{aligned}
    \end{equation}
    Combining equation (\ref{eq:4.12.1}) and Lemma \ref{lem:Stability and Generalization}, Part (c) completes the proof of Theorem \ref{the:3}.
\end{proof}

\subsection{Proof of Lemma \ref{lem:Link with Initialization Point}}
In this subsection, we prove Lemma  \ref{lem:Link with Initialization Point}, which shows that the expected gradient norms \( \mathbb{E}_S\left[\|\nabla_{\boldsymbol{\theta}} R_S(\boldsymbol{\theta}_t)\|_2\right] \) are influence by the choice of the initialization point. We first state and prove the following key lemma.
\renewcommand{\thetheorem}{C.9}
\begin{lemma}[Descent Lemma] \label{lem:C.9}
    Suppose that Assumption \ref{ass:Lipschitzness} and \ref{ass:smoothness} hold. Then, for \( \forall{\boldsymbol{\theta}_1,\boldsymbol{\theta}_2}\in\mathbb{R}^K \), and \( \boldsymbol{z}\in\mathcal{Z} \), we have
    \begin{equation*}
        \ell(f_{\boldsymbol{\theta}_1}(\boldsymbol{x}); y) - \ell(f_{\boldsymbol{\theta}_2}(\boldsymbol{x}); y) 
        \leq \langle \nabla_{\boldsymbol{\theta}} \ell(f_{\boldsymbol{\theta}_2}(\boldsymbol{x}); y), \boldsymbol{\theta}_1 - \boldsymbol{\theta}_2 \rangle 
        + \frac{\kappa}{2} \|\boldsymbol{\theta}_1 - \boldsymbol{\theta}_2\|_2^2,
    \end{equation*}
    where \( \kappa=\alpha_\ell K \|O\| + \sqrt{2} \nu_\ell K \|O\|^2 \).
\end{lemma}
\begin{proof}
    Let \( \Tilde{\boldsymbol{\theta}}\) be a point in the line segment of \( \boldsymbol{\theta}_1 \) and \( \boldsymbol{\theta}_2 \), \( \Tilde{\boldsymbol{\theta}}(u)=\boldsymbol{\theta}_2+u(\boldsymbol{\theta}_1-\boldsymbol{\theta}_2) \), then
    \begin{equation*}
    \begin{aligned}
        \ell(f_{\boldsymbol{\theta}_1}(\boldsymbol{x}); y) - \ell(f_{\boldsymbol{\theta}_2}(\boldsymbol{x}); y) 
        &= \int_0^1 \langle \boldsymbol{\theta}_1 - \boldsymbol{\theta}_2, 
        \nabla_{\boldsymbol{\theta}} \ell(f_{\Tilde{\boldsymbol{\theta}}(u)}(\boldsymbol{x}); y) \rangle \, du \\
        &= \int_0^1 \langle \boldsymbol{\theta}_1 - \boldsymbol{\theta}_2, 
        \nabla_{\boldsymbol{\theta}} \ell(f_{\boldsymbol{\theta}_2}(\boldsymbol{x}); y) 
        + \nabla_{\boldsymbol{\theta}} \ell(f_{\Tilde{\boldsymbol{\theta}}(u)}(\boldsymbol{x}); y) 
        - \nabla_{\boldsymbol{\theta}} \ell(f_{\boldsymbol{\theta}_2}(\boldsymbol{x}); y) \rangle \, du \\
        &= \langle \nabla_{\boldsymbol{\theta}} \ell(f_{\boldsymbol{\theta}_2}(\boldsymbol{x}); y), 
        \boldsymbol{\theta}_1 - \boldsymbol{\theta}_2 \rangle + \int_0^1 \langle \boldsymbol{\theta}_1 - \boldsymbol{\theta}_2, 
        \nabla_{\boldsymbol{\theta}} \ell(f_{\Tilde{\boldsymbol{\theta}}(u)}(\boldsymbol{x}); y) 
        - \nabla_{\boldsymbol{\theta}} \ell(f_{\boldsymbol{\theta}_2}(\boldsymbol{x}); y) \rangle \, du \\
        &\leq \langle \nabla_{\boldsymbol{\theta}} \ell(f_{\boldsymbol{\theta}_2}(\boldsymbol{x}); y), 
        \boldsymbol{\theta}_1 - \boldsymbol{\theta}_2 \rangle + \int_0^1 \|\boldsymbol{\theta}_1 - \boldsymbol{\theta}_2\| 
        \cdot \|\nabla_{\boldsymbol{\theta}} \ell(f_{\Tilde{\boldsymbol{\theta}}(u)}(\boldsymbol{x}); y) 
        - \nabla_{\boldsymbol{\theta}} \ell(f_{\boldsymbol{\theta}_2}(\boldsymbol{x}); y)\| \, du.
    \end{aligned}
    \end{equation*}
    According to Lemma \ref{lem:C.2}, we further get
     \begin{equation*}
    \begin{aligned}
        \ell(f_{\boldsymbol{\theta}_1}(\boldsymbol{x}); y) - \ell(f_{\boldsymbol{\theta}_2}(\boldsymbol{x}); y) 
        &\leq \langle \nabla_{\boldsymbol{\theta}} \ell(f_{\boldsymbol{\theta}_2}(\boldsymbol{x}); y), 
        \boldsymbol{\theta}_1 - \boldsymbol{\theta}_2 \rangle 
        + \int_0^1 \|\boldsymbol{\theta}_1 - \boldsymbol{\theta}_2\| \cdot \kappa \|\Tilde{\boldsymbol{\theta}}(u) - \boldsymbol{\theta}_2\| \, du \\
        &= \langle \nabla_{\boldsymbol{\theta}} \ell(f_{\boldsymbol{\theta}_2}(\boldsymbol{x}); y), 
        \boldsymbol{\theta}_1 - \boldsymbol{\theta}_2 \rangle 
        + \int_0^1 \|\boldsymbol{\theta}_1 - \boldsymbol{\theta}_2\| \cdot \kappa u \|\boldsymbol{\theta}_1 - \boldsymbol{\theta}_2\| \, du \\
        &= \langle \nabla_{\boldsymbol{\theta}} \ell(f_{\boldsymbol{\theta}_2}(\boldsymbol{x}); y), 
        \boldsymbol{\theta}_1 - \boldsymbol{\theta}_2 \rangle 
        + \kappa \|\boldsymbol{\theta}_1 - \boldsymbol{\theta}_2\|^2 \int_0^1 u \, du \\
        &= \langle \nabla_{\boldsymbol{\theta}} \ell(f_{\boldsymbol{\theta}_2}(\boldsymbol{x}); y), 
        \boldsymbol{\theta}_1 - \boldsymbol{\theta}_2 \rangle 
        + \frac{\kappa}{2} \|\boldsymbol{\theta}_1 - \boldsymbol{\theta}_2\|^2.
    \end{aligned}
    \end{equation*}
    This completes the proof of Lemma \ref{lem:C.9}.
\end{proof}

We are now ready to prove Lemma \ref{lem:Link with Initialization Point}.
\renewcommand{\thetheorem}{\ref{lem:Link with Initialization Point}}
\begin{lemma}[Link with Initialization Point]
     Suppose that Assumption \ref{ass:Lipschitzness}, \ref{ass:smoothness}, and \ref{ass:Variance} hold. Let \( A(S) \) be the QNN model trained on the dataset \( S\in\mathcal{Z}^m \) using the SGD algorithm with step sizes \( \eta_t\leq 1/\kappa
     \) for \( T \) iterations, then the following bound holds
     \begin{equation*}
        \begin{aligned}
            &\sum_{t=0}^{T-1}\eta_t\mathbb{E}_S\left[\|\nabla_{\boldsymbol{\theta}} R_S(\boldsymbol{\theta}_t)\|_2\right] \\
            &\leq 2\sqrt{\left(\sum_{t=0}^{T-1}\eta_t\right)\left(R_S(\boldsymbol{\theta}_0)-R_S(\boldsymbol{\theta}^*)+\frac{\kappa\sigma^2}{2}\sum_{t=0}^{T-1}\eta_t^2\right)},
        \end{aligned}
    \end{equation*}
    where \( \boldsymbol{\theta}^* \) is the empirical risk minimizer of \( R_S(\boldsymbol{\theta})\), \( \kappa \) is defined by equation (\ref{eq:k}).
\end{lemma}
\begin{proof}
    We can bound \( \sum_{t=0}^{T-1}\eta_t\mathbb{E}_S\left[\|\nabla_{\boldsymbol{\theta}} R_S(\boldsymbol{\theta}_t)\|_2\right] \) as follows
    \begin{equation}\label{eq:4.15.1}
    \begin{aligned}
         \sum_{t=0}^{T-1}\eta_t\mathbb{E}_S\left[\sqrt{\|\nabla_{\boldsymbol{\theta}} R_S(\boldsymbol{\theta}_t)\|_2^2}\right] 
         &\leq \sum_{t=0}^{T-1}\frac{\left(1-\frac{\eta_t\kappa}{2}\right)}{\left(1-\frac{\eta_t\kappa}{2}\right)}\cdot\eta_t\sqrt{\mathbb{E}_S\left[\|\nabla_{\boldsymbol{\theta}} R_S(\boldsymbol{\theta}_t)\|_2^2\right]}  \\
         &\leq 2\sum_{t=0}^{T-1}\left(\eta_t-\frac{\eta_t^2\kappa}{2}\right)\sqrt{\mathbb{E}_S\left[\|\nabla_{\boldsymbol{\theta}} R_S(\boldsymbol{\theta}_t)\|_2^2\right]} \\
         &=\frac{2\sum_{t=0}^{T-1}\left(\eta_t-\frac{\eta_t^2\kappa}{2}\right)}{\sum_{t=0}^{T-1}\left(\eta_t-\frac{\eta_t^2\kappa}{2}\right)}\sum_{t=0}^{T-1}\left(\eta_t-\frac{\eta_t^2\kappa}{2}\right)\sqrt{\mathbb{E}_S\left[\|\nabla_{\boldsymbol{\theta}} R_S(\boldsymbol{\theta}_t)\|_2^2\right]} \\
         &\leq 2\sqrt{\sum_{t=0}^{T-1}\eta_t}\sqrt{\sum_{t=0}^{T-1}\left(\eta_t-\frac{\eta_t^2\kappa}{2}\right)\mathbb{E}_S\left[\|\nabla_{\boldsymbol{\theta}} R_S(\boldsymbol{\theta}_t)\|_2^2\right]},
    \end{aligned}
    \end{equation}
    where the first and last inequality use the Jensen’s inequality, and the second inequality uses the condition \( \eta_t\leq1/\kappa \).
    
    Then, we focus on the term \( \sqrt{\sum_{t=0}^{T-1}\left(\eta_t-\frac{\eta_t^2\kappa}{2}\right)\mathbb{E}_S\left[\|\nabla_{\boldsymbol{\theta}} R_S(\boldsymbol{\theta}_t)\|_2^2\right]} \). By Lemma \ref{lem:C.9} and the update rule of the SGD, we obtain
    \begin{equation*}
    \begin{aligned}
        R_S(\boldsymbol{\theta}_{t+1}) - R_S(\boldsymbol{\theta}_t)
        &\leq \langle \boldsymbol{\theta}_{t+1} - \boldsymbol{\theta}_t, \nabla_{\boldsymbol{\theta}} R_S(\boldsymbol{\theta}_t) \rangle 
        + \frac{\kappa}{2} \|\boldsymbol{\theta}_{t+1} - \boldsymbol{\theta}_t\|_2^2 \\
        &= \langle -\eta_t \nabla_{\boldsymbol{\theta}} \ell(f_{\boldsymbol{\theta}_t}(\boldsymbol{x}_{i_t}); y_{i_t}), 
        \nabla_{\boldsymbol{\theta}} R_S(\boldsymbol{\theta}_t) \rangle 
        + \frac{\kappa \eta_t^2}{2} \|\nabla_{\boldsymbol{\theta}} \ell(f_{\boldsymbol{\theta}_t}(\boldsymbol{x}_{i_t}); y_{i_t})\|_2^2 \\
        &= \langle -\eta_t \nabla_{\boldsymbol{\theta}} \ell(f_{\boldsymbol{\theta}_t}(\boldsymbol{x}_{i_t}); y_{i_t}), 
        \nabla_{\boldsymbol{\theta}} R_S(\boldsymbol{\theta}_t) \rangle \\
        &\quad + \frac{\kappa \eta_t^2}{2} \big( 
        \|\nabla_{\boldsymbol{\theta}} \ell(f_{\boldsymbol{\theta}_t}(\boldsymbol{x}_{i_t}); y_{i_t}) 
        - \nabla_{\boldsymbol{\theta}} R_S(\boldsymbol{\theta}_t)\|_2^2 + \|\nabla_{\boldsymbol{\theta}} R_S(\boldsymbol{\theta}_t)\|_2^2 \\
        &\quad - 2 \langle \nabla_{\boldsymbol{\theta}} \ell(f_{\boldsymbol{\theta}_t}(\boldsymbol{x}_{i_t}); y_{i_t}) 
        - \nabla_{\boldsymbol{\theta}} R_S(\boldsymbol{\theta}_t), 
        \nabla_{\boldsymbol{\theta}} R_S(\boldsymbol{\theta}_t) \rangle 
        \big) \\
        &= -\left(\eta_t + \eta_t^2 \kappa \right) 
        \langle \nabla_{\boldsymbol{\theta}} \ell(f_{\boldsymbol{\theta}_t}(\boldsymbol{x}_{i_t}); y_{i_t}), 
        \nabla_{\boldsymbol{\theta}} R_S(\boldsymbol{\theta}_t) \rangle + \frac{3\eta_t^2\kappa}{2}\|\nabla_{\boldsymbol{\theta}} \ell(f_{\boldsymbol{\theta}_t}(\boldsymbol{x}_{i_t}); y_{i_t}) 
        - \nabla_{\boldsymbol{\theta}} R_S(\boldsymbol{\theta}_t)\|_2^2
    \end{aligned}
    \end{equation*}
    Using the Assumption \ref{ass:Variance} and take an average over \( i_t\in[m] \), we have
   \begin{equation*}
    \begin{aligned}
        \left(\eta_t - \frac{\eta_t^2 \kappa}{2}\right)
        \mathbb{E}_S\left[\|\nabla_{\boldsymbol{\theta}} R_S(\boldsymbol{\theta}_t)\|_2^2\right]
        &\leq R_S(\boldsymbol{\theta}_t) - R_S(\boldsymbol{\theta}_{t+1}) 
        + \frac{\eta_t^2 \kappa}{2} 
        \mathbb{E}_S\left[\|\nabla_{\boldsymbol{\theta}} \ell(f_{\boldsymbol{\theta}_t}(\boldsymbol{x}_{i_t}); y_{i_t}) 
        - \nabla_{\boldsymbol{\theta}} R_S(\boldsymbol{\theta}_t)\|_2^2\right] \\
        &\leq R_S(\boldsymbol{\theta}_t) - R_S(\boldsymbol{\theta}_{t+1}) + \frac{\eta_t^2 \kappa}{2} \sigma^2.
    \end{aligned}
    \end{equation*}
    We can apply the above inequality recursively and derive
    \begin{equation*}
    \begin{aligned}
        \sum_{t=0}^{T-1}\left(\eta_t-\frac{\eta_t^2\kappa}{2}\right)\mathbb{E}_S\left[\|\nabla_{\boldsymbol{\theta}} R_S(\boldsymbol{\theta}_t)\|_2^2\right]
        &\leq  R_S(\boldsymbol{\theta}_0) - R_S(\boldsymbol{\theta}^*) + \frac{\kappa\sigma^2}{2}\sum_{t=0}^{T-1} \eta_t^2.
    \end{aligned}
    \end{equation*}
    Plugging it into equation (\ref{eq:4.15.1}) completes the proof of Lemma \ref{lem:Link with Initialization Point}.
\end{proof}

\end{document}